\documentclass{article}
\usepackage[utf8]{inputenc} % allow utf-8 input
\usepackage[T1]{fontenc}    % use 8-bit T1 fonts

\usepackage{url}            % simple URL typesetting
\usepackage{booktabs}       % professional-quality tables
\usepackage{amsfonts}       % blackboard math symbols
\usepackage{nicefrac}       % compact symbols for 1/2, etc.
\usepackage{microtype}      % microtypography
\usepackage{amsmath}
\usepackage[normalem]{ulem}
\usepackage[dvipsnames]{xcolor}
\usepackage{multirow}
\usepackage{array}
\usepackage{lmodern}

\usepackage{algorithm}
\usepackage{algpseudocode}

\usepackage[shortlabels]{enumitem}
\usepackage[font=small,labelfont=bf]{caption}
\usepackage{subcaption}
\usepackage[percent]{overpic}
\usepackage[margin=1.2in]{geometry}
\usepackage{amssymb, amsmath}
\usepackage{amsthm}
\usepackage{dsfont}
\usepackage{stmaryrd}
\usepackage{caption}
\usepackage[colorlinks=true, linkcolor=blue, citecolor=blue]{hyperref}
\usepackage{wrapfig}

\usepackage[linecolor=blue!60!,backgroundcolor=blue!10!,textwidth=2.5cm,textsize=scriptsize,disable=true]{todonotes}

\usepackage{color} 

\definecolor{airforceblue}{rgb}{0.36, 0.54, 0.66}

\input{macrogilles_main_modif_icml.sty}

\usepackage[authoryear,round]{natbib}
\usepackage{xcolor}

\newcommand{\HFF}{\cH_\Gamma}
\newcommand{\Spx}{\cS}
\newcommand{\omvect}{{\bm{\om}}}
\newcommand{\estmu}{\wh{\mu}}
\newcommand{\muNE}{\estmu}

\newcommand{\sigmac}{\sigma_{\texttt{c}}}

\newcommand{\MMD}{\mathrm{MMD}}

\newcommand{\comm}[3][inline]{\todo[#1, size=\tiny]{#2: #3}}

\newcommand{\batodo}[2][noinline]{\comm[color=magenta,#1]{BL}{#2}}
\newcommand{\abtodo}[2][noinline]{\comm[color=green,#1]{AB}{#2}}

\theoremstyle{plain}
\newtheorem{theorem}{Theorem}[section]
\newtheorem{proposition}[theorem]{Proposition}
\newtheorem{lemma}[theorem]{Lemma}
\newtheorem{corollary}[theorem]{Corollary}
\theoremstyle{definition}
\newtheorem{definition}[theorem]{Definition}
\newtheorem{assumption}[theorem]{Assumption}
\theoremstyle{remark}
\newtheorem{remark}[theorem]{Remark}
\newtheorem{example}[theorem]{Example}

\title{
	\vspace{-3.5em}
	\noindent\rule[0.5ex]{\linewidth}{1pt} 
	{\LARGE Adaptive Personalized Federated Learning \\ via Multi-task Averaging of Kernel Mean Embeddings} \
	\noindent\rule[0.5ex]{\linewidth}{1pt}\vspace{-0.5em}
}

\author{
	{\large Jean-Baptiste Fermanian\textsuperscript{\textdagger}} \qquad 
	{\large Batiste Le Bars*} \qquad 
	{\large Aurélien Bellet\textsuperscript{\textdagger}} \\[0.4cm]
	\textsuperscript{\textdagger} {\normalsize PreMeDICal, Inria, Idesp, Inserm, University of Montpellier} \\[0.1cm]
	* {\normalsize Univ. Lille, Inria, CNRS, Centrale Lille, UMR 9189, CRIStAL, F-59000 Lille} \\[0.1cm]
}
\date{}

\begin{document}

\maketitle

\begin{abstract}
   Personalized Federated Learning (PFL) enables a collection of agents to collaboratively learn individual models without sharing raw data. We propose a new PFL approach in which each agent optimizes a weighted combination of all agents' empirical risks, with the weights learned from data rather than specified a priori. The novelty of our method lies in formulating the estimation of these collaborative weights as a kernel mean embedding estimation problem with multiple data sources, leveraging tools from multi-task averaging to capture statistical relationships between agents. This perspective yields a fully adaptive procedure that requires no prior knowledge of data heterogeneity and can automatically transition between global and local learning regimes. By recasting the objective as a high-dimensional mean estimation problem, we derive finite-sample guarantees on local excess risks for a broad class of distributions, explicitly quantifying the statistical gains of collaboration. To address communication constraints inherent to federated settings, we also propose a practical implementation based on random Fourier features, which allows one to trade communication cost for statistical efficiency. Numerical experiments validate our theoretical results.
\end{abstract}

\section{Introduction}

Despite the growing volume of data collected worldwide, many application domains, such as medicine and ecology, remain fundamentally data-limited. In such contexts, a major scientific challenge is to learn effective models from data originating from multiple sources that are often heterogeneous and biased, yet share enough similarities to be jointly exploited. 
In healthcare, they may arise from different hospitals observing distinct patient populations or operating under variations in medical devices and clinical protocols \citep{rieke2020future,xu2021federated,nguyen2022federated}. % dasaradharami2023comprehensive,moshawrab2023reviewing,shahin2024federated}.
Similar challenges also appear in astrophysics, where data are collected across instruments, wavelengths, or angular resolutions that vary between observation centers \citep{elmahallawy2022asyncfleo,razmi2022board,chen2022satellite}.

When all data are directly accessible, learning from heterogeneous sources is commonly framed as multi-task learning \citep{caruana1997multitask,zhang2021survey}. We instead consider a more challenging and increasingly relevant setting in which data are distributed across multiple data owners (e.g.~hospitals), hereafter referred to as \emph{agents}, that seek to collaborate without sharing their raw data. Such constraints may arise because data are too sensitive to be shared or due to transmission costs, bandwdith or storage limitations, institutional constraints, or legal barriers. Over the past few years, this setting has been addressed within the framework of Federated Learning (FL, \citealp{kairouz2021advances}). Early FL methods primarily focused on learning a single global model that performs well on average across all decentralized data. However, due to the heterogeneity that naturally arises across agents in decentralized environments, this one-size-fits-all objective has quickly been recognized as a limitation. To address this issue, the paradigm of \emph{Personalized} FL (PFL) has emerged. Similarly to multi-task learning, PFL aims to learn a agent-specific models while still enabling collaboration between them. The central challenge is then to manage inter-agent heterogeneity and control the bias induced by leveraging data from other agents.

Numerous methods have been proposed to address this problem (see Section~\ref{sec:related_works} for an overview). Most of them enable collaboration by assuming some structure among agents, and in some cases even require that this structure is known—for example, that all local models are close to a global model, that agents form a fixed number of clusters, or that each model can be expressed as a linear combination of a local and a shared global model. These assumptions, however, are often violated in practice, limiting the methods' effectiveness. More generally, existing approaches are largely heuristic and provide no generalization guarantees demonstrating a statistical benefit of collaboration over learning in isolation. In this work, we address these gaps by proposing an approach that requires no prior knowledge of agents' heterogeneity, automatically adapts to their underlying structure, and comes with generalization guarantees that quantify the advantage of collaboration. More precisely:

\begin{itemize}
\item  We formulate the PFL problem as learning a mixture of the agents' data distributions. Assuming the loss lies to a \emph{Reproducing Kernel Hilbert Space} (RKHS), we link the excess risk to the \textit{Maximum Mean Discrepancy} (MMD) between a target agent's distribution and the estimated mixture (Eq.~\ref{eq:approx_err_to_excess_risk} and Lemma \ref{lem:excessrisk_mmd}). Learning the mixture weights by minimizing the MMD then amounts to aggregating the \emph{Kernel Mean Embeddings} (KMEs) of the agents' distributions.

\item Since KMEs are high-dimensional means, we leverage the Q-aggregation mean estimation method of \citet{blanchard2024estimation} to estimate the mixture weights. A novel theoretical analysis (Theorem \ref{prop:kme_aggregation} and Corollary~\ref{thm:kernel}) demonstrates the resulting statistical gains in terms of the excess risk evaluated on the target agent's distribution.

\item As sharing KMEs directly would typically require sharing all the raw data, conflicting with federated learning requirement, we propose a practical method based on \emph{random Fourier features}, for which we derive theoretical guarantees quantifying the trade-off between communication costs and statistical efficiency (Theorem~\ref{cor:kmerff}).

\item We validate our approach empirically on synthetic and real-world data, showing that it effectively adapts to the heterogeneity across agents. 
\end{itemize}
\section{Preliminaries}

\textbf{Notations.} For an integer $B \in \mathbb{N}$, let $\intr{B} := \set{ 1,\ldots,B}$. We denote by $\Spx_B:= \set[1]{\omvect \in [0,1]^B : \sum_{i=1}^{B} \omega_i =1}$ the $B$-simplex. For $a,b\in \mbr$, $a\vee b := \max(a,b)$.

\subsection{Setting and Objective}
\label{sec:setting}
\abtodo{un peu bizarre $B$ pour les agents, en plus on utilise $k$ comme indice, peut être au moins utiliser $K$?}\batodo{+1}\todo{pas de souci pour changer, faut juste le faire dans toutes les preuves aussi}
We consider a setting with $B$ agents, each having access to its own local dataset. The dataset of an agent $k \in \intr{B}$, denoted by
$Z_\bullet^{(k)} = \set[1]{ Z_i^{(k)}}_{i=1}^{n_k}$, consists of $n_k$ i.i.d.~samples drawn from a distribution $\mbp_k$ with support in $\cZ$. We assume no prior similarity between $\mbp_k$ and the distributions of other agents. Our goal is to leverage this multiplicity of sources to improve the model learned for a target agent, say $k=1$, beyond what could be achieved using only its local data. This kind of setting is referred as \textit{all-for-one} \citep{even2022sample}. Formally, we study the excess risk of a learned model~$\widehat{\theta}$: \looseness=-1
\begin{equation}\label{eq:def_gen_error}
\e{ \cR_1(\widehat{\theta})} - \cR_1(\theta^*_1)\;,
\end{equation}
where $\cR_k(\theta) = \ee{Z\sim\mbp_k}{\ell_\theta(Z)}$ denotes the risk associated with a loss function $\ell_\theta : \cZ \to\mbr$ and $\theta^*_k \in \argmin_{\theta \in \Theta} \cR_k(\theta)$ is an optimal model for agent $k$.

In a classical learning setting, we cannot directly minimize the population risk $\cR_1$, so a standard approach is to minimize its empirical counterpart using the local dataset $Z_\bullet^{(1)}$. The empirical risk is defined as $\widehat{\cR}_1(\theta) = \frac{1}{n_1}\sum_{i=1}^{n_1}\ell_\theta(Z_i^{(1)})$, and the excess risk~\eqref{eq:def_gen_error} of its minimizer is well studied, typically scaling at $\mathcal{O}(n_1^{-\nicefrac{1}{2}})$ (see e.g. \citealp{bach2024learning}). 
In this work, we address the PFL problem by minimizing a weighted empirical risk instead:
\begin{equation}\label{eq:weighted-risks}
\textstyle\widehat{\cR}_{\wh{\omvect}}(\theta) := \sum_{k=1}^{B}\widehat{\omega}_k \widehat{\cR}_k(\theta)\;,
\end{equation}
where the weights $\widehat{\omvect}\in\Spx_B$ are themselves to be learned. Our goal is to show that minimizing this weighted risk can improve the excess risk compared to minimizing $\widehat{\cR}_1$. We emphasize that, given the weights, our focus is not on minimizing~\eqref{eq:weighted-risks} in a FL setting, many algorithms exist for that, but on computing the weights in $\widehat{\omvect}$ and deriving statistical bounds on the excess risk \eqref{eq:def_gen_error}. Recall that, if the optimization is exact, i.e., $\widehat{\theta} \in \arg\min_\theta \widehat{\mathcal{R}}_{\widehat{\boldsymbol{\omega}}}(\theta)$, the excess risk is controlled by (twice) the generalization error:\looseness=-1
\begin{equation}\label{eq:approx_err_to_excess_risk}
\e{ \cR_{\widehat{\omvect}}(\widehat{\theta})} - \cR_1(\theta^*_1) \leq 2\e[2]{ \sup_{\theta \in \Theta} \abs[1]{\wh{\cR}_{\wh{\omvect}}(\theta) - \cR_1(\theta)}}\,,
\end{equation}
which is the central quantity analyzed in this work. 
\begin{remark}[Optimization error]
In the case of imperfect minimization of $\wh{\cR}_{\widehat{\omvect}}$, an additional optimization error term, typically dependent on the regularity of the loss $\theta\mapsto\ell_\theta$ and the number of iterations, would be added to the generalization error in~\eqref{eq:approx_err_to_excess_risk}. However, controlling this is beyond the scope of this work, which focuses on the statistical error.
\end{remark}

\subsection{RKHS and KME}\label{ssec:rkhsandkme}

We briefly recall some tools from kernel methods, which are central to our approach. A kernel $\kappa : \cZ \times \cZ \to \mbr$ is positive definite if for any finite sequences $a_i \in \mbr$ and $z_i \in \cZ$, $\sum_{ij} a_ia_j \kappa(z_i,z_j)\geq 0$. For any positive definite kernel, there exists a unique Hilbert space $\cH \subset \set{f,f: \cZ \to \mbr}$ such that for any $h \in \cH$ and $z\in \cZ$, $\inner{h, \kappa(z,\cdot) }_\cH = h(z)$ \citep{aronszajn1950theory}. This space is called a RKHS and corresponds to the completion of the span of $\set{\phi_\kappa(z): z\in \cZ}$, where $\phi_\kappa(z)(\cdot) = \kappa(z,\cdot)$ is called the feature map.

For a distribution $\mbp$ on $\cZ$, the KME $\mu_\mbp\in\cH$ \citep{smola2007hilbert} is defined as
\begin{equation}\label{eq:def_kme}
\mu_\mbp(\cdot) =  \ee{Z\sim \mbp}{ \phi_\kappa(Z)} = \ee{Z\sim \mbp}{ \kappa(Z,\cdot)}\,. 
\end{equation}
Then, for any function $h \in \cH$, $\ee{Z\sim \mbp}{ h(Z)} = \inner{ h, \mu_\mbp}_\cH$. 
KMEs allow defining a distance between distributions, the MMD \citep{gretton2012kernel}:% defined as 
\begin{equation}\label{eq:MMD_def}
\MMD_\kappa(\mbp,\mbq) = \norm{\mu_\mbp - \mu_\mbq}_\cH 
= \sup_{h\in \cH: \norm{h}_\cH=1} \ee{\substack{Z\sim \mbp, \\W\sim \mbq}}{ h(Z) - h(W)}.
\end{equation}
Computing the MMD only requires computing scalar products between KMEs, which can easily be estimated. Indeed, one can notice that
$\inner{\mu_\mbp,\mu_\mbq}_\cH = \ee{Z\sim \mbp,W\sim \mbq}{\kappa(Z,W)}$,
which is estimated by averaging kernel evaluations across pairs of samples from each distribution. In particular, the KME of $\mbp$ can be approximated by 
$\muNE_\mbp(\cdot) = \frac{1}{n} \sum_{i=1}^{n} \kappa(Z_i,\cdot)$,
where $(Z_i)_{i=1}^n$ are i.i.d. samples from $\mbp$. 
The mean squared error of this approximation corresponds to the MMD between the true and empirical distributions:
\begin{equation}\label{eq:mse_kme}
\e{\MMD^2_\kappa(\mbp,\widehat{\mbp})} = \e{\norm{
		\widehat{\mu}_\mbp - \mu_\mbp}^2_\cH} = \frac{\tr \Sigma(\mbp)}{n}\,,
		\end{equation}
		where $\Sigma(\mbp) : \cH \to \cH$ is the covariance operator of the pushforward of $\mbp$ into $\cH$ by $\phi_\kappa$ (see Definition~\ref{def:covop}). Note
		that the trace and operator norm of such operator are defined similarly to those in finite dimension (see Definition~\ref{def:trace_opnorm}).

		\subsection{Random Fourier Features}\label{sec:RFF}
		
		Computing inner products between KMEs can be costly for large datasets. A standard approach to reduce this cost is to approximate the RKHS $\cH$ using Random Fourier Features (RFF, \citealp{rahimi2007random,rahimi2008uniform,rahimi2008weighted}). 
		
		Let $\kappa : \cZ \times\cZ \to \mbr$ be a translation-invariant positive definite kernel (i.e. $\kappa(x,y) = \kappa(x-y)$). By Bochner's theorem \citep{bochner2005harmonic}, $\cH$ with feature map $\phi_\kappa$ can be approximated by a finite-dimensional RKHS $\HFF\subset \mbr^D$ with mapping $\phi_\Gamma$, using a distribution $p$ on $\cZ$. Specifically, draw $w_s\sim p$, $b_s \sim \cU([0,2\pi])$ for $s\in \intr{D}$, and define:
		\begin{equation}
\phi_\Gamma(z) = \textstyle\sqrt{\frac{2}{D}} \paren[1]{ \cos( \inner{ z, w_s}+ b_s )}_{s=1}^D\in\HFF.
\end{equation}
This mapping preserves the kernel in expectation:
\begin{equation}
\kappa(z,z') \! =\!\!\inner{\phi_\kappa(z) ,\! \phi_\kappa(z')}_\cH \!\! =\!  \ee{\Gamma}{ \inner{ \phi_\Gamma(z) ,\! \phi_\Gamma(z')}_{\mbr^D}}\,,
\end{equation}
where the expectation is taken over the random parameters $\Gamma=((w_s,b_s))_{s=1}^D$.
Although elements of $\HFF$ are represented as vectors in $\mbr^D$, they define functions from $\cZ$ to $\mbr$: if $h \in \HFF$, then
$h(z) = \inner{ h, \phi_\Gamma(z)}_{\mbr^D}$. Other properties of these objects are detailed in Appendix~\ref{APPsec:RFF}.

\section{Related Work}\label{sec:related_works}

\textbf{Personalized federated learning.} Addressing agent heterogeneity in FL through personalization has attracted significant interest in recent years, leading to a variety of approaches \citep{kulkarni2020survey,smith2017federated,sattler2020clustered,mansour2020three,t2020personalized,tan2022towards,cui2022collaboration,wu2023personalized}.

Many approaches rely on strong assumptions about the form of heterogeneity. Meta-learning, model interpolation, and fine-tuning methods, for instance, assume that a global model or shared representation provides a good starting point for all agents \citep{chen2018federated,DBLP:journals/corr/abs-1912-00818,fallah2020personalized,li2021ditto,deng2020adaptive,NEURIPS2020_187acf79}. Cluster-based methods assume that agents belong to a fixed number of clusters and learn one model per cluster \citep{sattler2020clustered,ghosh2020efficient,marfoq2021federated}. Kernel-based methods similarly assume shared structure, either a shared model with personalized combinations of base kernels \citep{m2022personalized} or a shared kernel with personalized models \citep{achituve2021personalized}. Such assumptions limit adaptability: if the underlying structure does not match these priors, these methods can fail entirely.\looseness=-1

To overcome these limitations, a more flexible line of work learns similarity or collaboration weights between agents, allowing adaptation to diverse heterogeneity patterns. For example, \citet{zantedeschi2020fully} regularize local objectives with pairwise weights $w_{ij}|\theta_i-\theta_j|$ that are learned jointly with the models, while CoBo \citep{hashemi2024cobo} learns such weights dynamically via bilevel optimization. In a related direction, \citet{pmlr-v258-kharrat25a} construct a collaboration graph that guides each client in selecting suitable collaborators. Our approach builds on this paradigm, focusing on the computation of collaboration weights in a statistically grounded manner.\looseness=-1

Despite extensive work, few methods provide generalization guarantees. Most focus on the convergence of optimization algorithms on training data, without quantifying the excess risk over the agents' population distributions. Some obtain bounds under stochastic optimization with fresh samples at each iteration \citep{even2022sample,scaman2024minimax,philippenko2025adaptive}, limiting applicability and preventing multiple passes over datasets of varying sizes. Moreover, these bounds typically assume oracle knowledge of collaboration weights or consider weights learned for very specific models, such as linear regression. A few works, like \citet{ding2022collaborative}, provide finite-sample generalization guarantees, but only for theoretical weights rather than weights learned in practice. In contrast, our approach establishes excess risk guarantees for weights actually estimated from data, without structural assumptions, explicitly demonstrating the statistical benefit of collaboration.\\

\noindent\textbf{High-dimensional multiple mean estimation.}
Our contribution builds on the classical problem of multiple mean estimation. Stein's paradox famously showed that the standard empirical averages are suboptimal for estimating multiple means, and that shrinkage estimators can improve estimation accuracy, especially in high-dimensional settings \citep{stein1956inadmissibility,james1961estimation,efron1972empirical,george1986minimax}. More recently, this problem has been revisited in multi-task learning \citep{martinez2013multi,feldman2014revisiting,duan2023adaptive} and in the estimation of KMEs \citep{muandet2014kernel}. From a high-dimensional perspective, \citet{marienwald2021high}, extended in \citet{blanchard2024estimation}, propose a KME aggregation method, which we leverage in this work. As detailed below, we relate the estimation of the weighted risk~\eqref{eq:weighted-risks} to learning a mixture over agents. We then use KME aggregation to estimate the corresponding weights, and quantify the resulting improvement.

\section{Personalized Learning as High-Dimensional Mean Estimation}

In this section, we reformulate the PFL objective introduced in Section~\ref{sec:setting}, namely, learning the weights $\widehat{\omvect}$ to minimize the excess risk, as a high-dimensional mean estimation problem in a RKHS with multiple data sources. To the best of our knowledge, this is the first work establishing a formal connection between these two problems, which enables us to transfer algorithms with strong statistical guarantees from the latter setting to PFL. We emphasize that our focus here is on a general approach for learning the weights in $\widehat{\omvect}$; the practical implementation in a federated learning context is deferred to Section~\ref{sec:rff}. Proofs are provided in Section~\ref{APPsec:proofs}.

\subsection{Controlling Generalization with MMD}
 
Recall from Eq.~\eqref{eq:weighted-risks} that our goal is to leverage the data distributions of all agents to learn a better model for a target agent, say agent~$1$, by estimating weights $\widehat{\omvect} \in \Spx_B$. 
Equivalently, the local empirical distribution $\widehat{\mbp}_1$ is replaced by a mixture of empirical distributions: 
\begin{equation} 
\textstyle\widehat{\mbp}(\widehat{\omvect})=\sum_{k=1}^{B} \widehat{\omega}_k \widehat{\mbp}_k.
\end{equation} 
The weights $\widehat{\omega}\in\Spx_B$ are chosen so that this mixture better approximates the target distribution $\mbp_1$ than $\hat {\mbp}_1$ alone. It remains to find a metric measuring this approximation, while controlling the target generalization error~\eqref{eq:approx_err_to_excess_risk}. To this aim, we assume that the loss functions $\ell_\theta$ belongs, up to a constant, to some RKHS $\cH$ (Assumption~\ref{ass:lossinrkhs}).

\begin{assumption}\label{ass:lossinrkhs}
For a kernel $\kappa$, for all $\theta \in \Theta$, $\exists c_\theta \in \mbr$ and $h_\theta \in \cH$ such that $\ell_\theta(z) = c_\theta + h_\theta(z)$ for $z
\in \cZ$.
\end{assumption}
The constant $c_\theta$  allows to cover a broader class of loss functions, including constant losses that are not contained in standard RKHSs such as the one induced by the Gaussian kernel.\looseness=-1

\begin{example}[Linear regression]\label{ex:linear_regression_rkhs}
The ridge loss $\ell_\theta: z \mapsto\paren{\inner{\alpha,x} + \beta -y }^2 + \lambda \norm{\theta}^2$, where $\theta =(\alpha,\beta)\in \mbr^{d+1}$ and $z=(x,y) \in \mbr^{d+1}$ satisfies Assumption~\ref{ass:lossinrkhs} for the polynomial kernel $\kappa(z,z') = \paren[1]{ \inner{z,z'}+1}^2$ with $c_\theta = \lambda \norm{\theta}^2$.%\batodo{Faire ref à preuve} 
\end{example}

This allows us to control the generalization error in function of the MMD (Eq.~\ref{eq:MMD_def}) of the two distributions $\widehat{\mbp}(\widehat{\omvect})$ and $\mbp_1$. 

\begin{lemma}\label{lem:excessrisk_mmd}
Under Assumption~\ref{ass:lossinrkhs}, for any learned weights $\widehat{\omega}$, we have:
\begin{equation}\label{eq:generror_to_mmd}
	\!	\e[2]{ \sup_{\theta \in \Theta} \abs[1]{\wh{\cR}_{\wh{\omvect}}(\theta) \!-\! \cR_1(\theta)}}
	\!\leq\! R_\Theta \e[2]{\MMD_\kappa\paren[1]{\mbp_1,\widehat{\mbp}(\widehat{\omvect})}},
\end{equation}
where $R_\Theta = \sup_{\theta} \norm{h_\theta}_\cH$.
Moreover, if for some $r>0$, $ \set{h\in \cH : \norm{h}_\cH =r}\subset \set{ h_\theta}_{\theta \in \Theta}$ , then:
\begin{equation}\label{eq:naive_lowerbound}
	\e[2]{ \sup_{\theta \in \Theta} \paren{\widehat{\cR}_1(\theta) - \cR_1(\theta)}^2} \geq r^2 \frac{\tr \Sigma_1}{n_1}\,,
\end{equation}
where $\Sigma_1 = \Sigma(\mbp_1)$ is the covariance of $\mbp_1$ in $\cH$.
\end{lemma}

This lemma, combined with Eq.~\eqref{eq:approx_err_to_excess_risk}, shows that the control of the MMD distance between the mixture and the target distribution directly controls the excess risk.
The lower bound~\eqref{eq:naive_lowerbound} indicates that to improve upon the naive local estimator, the mixture must achieve an MMD of at most $\sqrt{\tr \Sigma_1/n_1}$; otherwise, training solely on the local empirical distribution is preferable.\\

\textbf{Link with high-dimensional multiple mean estimation.} As discussed in Section~\ref{ssec:rkhsandkme}, the MMD corresponds to the distance between KMEs (Eq.~\ref{eq:MMD_def}). Since the KME of a mixture of distributions is the convex combination of the individual KMEs, controlling the MMD in \eqref{eq:generror_to_mmd} reduces to:
\begin{equation}
\label{eq:mmd-kme}
\e[2]{\MMD^2_\kappa\paren[1]{\mbp_1,\widehat{\mbp}(\widehat{\omvect})}} = \e[2]{\norm[2]{ \sum_{k=1}^{B} \widehat{\omega}_k \widehat{\mu}_k - \mu_1}^2_\cH}\,,
\end{equation}
where $\widehat{\mu}_k = \widehat{\mu}_{\mbp_k}$ is the empirical KME of agent $k$. In conclusion, Eq.~\eqref{eq:mmd-kme} indicates that finding the weights $\widehat{\omvect}$ minimizing the upper-bound in Lemma~\ref{lem:excessrisk_mmd} is equivalent to estimating the KME $\mu_1$ with the aggregated empirical KMEs. Those objects being high-dimensional means, we transformed our initial objective to a high-dimensional mean estimation problem with multiple data sources. In the next section, we leverage recent work of \citet{blanchard2024estimation} that have tackled this problem in general settings.

\subsection{Learning the Mixture Weights by Q-Aggregation}

To estimate the mixture weights $\wh{\omvect}$, we adopt the Q-aggregation method of \citet{blanchard2024estimation}, originally developed for multiple estimation of high-dimensional means, which may in particular correspond to KMEs of different distributions. We emphasize that while this method is not new, the theoretical results that follow are novel as they are derived for the specific KME estimation setting.

The key insight behind this approach is that determining whether two high dimensional means are close to each other is often easier than estimating them precisely \citep{baraud2002non,blanchard2018minimax,blanchard2019nonasymptotic}. In the infinite-dimensional setting, the ``high-dimensional effect'' is measured through a notion of \emph{effective dimension} $d^e(\mbp)$ of the distribution: 
\begin{equation}
d^e(\mbp) := \frac{\tr \Sigma(\mbp)}{\norm{ \Sigma(\mbp)}_{op}}\,, \; d_1^e := d^e(\mbp_1),
\end{equation} 

This notion is also referred to as intrinsic dimension \citep{hsu2012tail,tropp2015introduction} or effective rank \citep{koltchinskii2016asymptotics}. For isotropic distributions, it coincides with the actual support dimension, but it remains well-defined in infinite-dimensional settings. Intuitively, $d^e(\mbp)$ quantifies the degrees of freedom of the distribution.

The Q-aggregation method is detailed in Algorithm~\ref{alg:Qaggregation}. It relies on an unbiased estimation $\widehat{L}_1(\omvect)$ of the mean squared error  (in our case, the MMD) of the convex aggregation of empirical means (in our case, the KMEs, see Eq.~\ref{eq:mmd-kme}). The weights $\widehat{\omvect}$ are then obtained by minimizing this empirical error with a penalization term that accounts for the high-dimensional effect, depending on the (estimated) covariance $\widehat{\Sigma}_1$ of the target distribution. Intuitively, this penalty ensures the error is not underestimated and is controlled uniformly. 
Algorithm~\ref{alg:Qaggregation} is stated for a general Hilbert space $\cH$ and can be implemented whenever scalar products in $\cH$ are computable, which holds both in finite dimension and in RKHS. The optimization reduces to quadratic form $\omega^T\wh{A}\omega+ \inner{\beta,\omega}$ over the simplex $\Spx_B$, where $\wh{A}$ is the Gram matrix of the vectors $\wh{\nu}_k - \wh{\nu}_1$. This minimization can performed, for example, via exponential gradient descent \citep{kivinen1997exponentiated}. For clarity, additional implementation details are deferred to Appendix~\ref{APPsec:qaggreg_details}.

\begin{algorithm}[h]
\caption{Q-aggregation method}
\label{alg:Qaggregation}
\begin{algorithmic}
	\State {\bfseries Input:} Hilbert space $(\cH,\inner{\cdot,\cdot}_\cH)$, empirical means $\widehat{\nu}_k$ of each agent, dataset $(\Phi_i^{(1)})_{i=1}^{n_1}$ of targeted agent $1$, bound on data $M$, $C_Q,C_P>0$
	\State {\bf Let $\wh{\Sigma}_1$ be the empirical covariance operator:}\\ 
	$\widehat{\Sigma}_1(\cdot)=\frac{1}{n_1-1} \sum_{i=1}^{n_1}\inner[1]{ \Phi_i^{(1)} - \widehat{\nu}_1, \cdot}_\cH (\Phi_i^{(1)} - \widehat{\nu}_1)$
	\State  {\bf Compute empirical error:} for $\omvect \in \Spx_B$\\
	$\wh{L}_1(\omvect) := \norm[1]{ \sum_{k=1}^{B} \omega_k \widehat{\nu}_k - \widehat{\nu}_1 }^2_\cH + 2\om_1 \frac{\tr \widehat{\Sigma}_1}{n_1}$
	\State {\bf Compute penalization terms:} for $\omvect \in \Spx_B$\\
	$\widehat{Q}_1(\omvect) := \frac{1}{\sqrt{n_1}}\sum_{k=2}^{B} \omega_k  \inner[1]{\widehat{\nu}_1 - \widehat{\nu}_k, \widehat{\Sigma}_1( \widehat{\nu}_1 - \widehat{\nu}_k)}^{\nicefrac{1}{2}}_\cH $\\
	$\wh{P}_1(\omvect) := \frac{M}{n_1} \sum_{k=2}^B \omega_k \norm{\widehat{\nu}_k-\widehat{\nu}_1}_\cH$
	\State {\bf Compute weights:}\\
	$\widehat{\omvect} \in \arg\min_{\omvect \in \Spx_B}
	\paren{\wh{L}_1(\omvect) + C_Q \wh{Q}_1(\omvect) + C_P \wh{P}_1(\omvect)}$
	\State {\bf Output} Return $\widehat{\omvect}$
\end{algorithmic}
\end{algorithm}

\citet{blanchard2024estimation} showed that this algorithm achieves an optimal trade-off, as restated in Theorem~\ref{thm:oracle_ineq_BS}. In the context of KME estimation, their result can be further refined,
%\batodo{Briefly discuss what this refinement brought? \JB{c'est juste plus lisible, cf appendix}}
see Theorem~\ref{prop:kme_aggregation} below. This refinement is possible because, unlike the general high-dimensional mean estimation setting, where the mean and covariance of a distribution are separate degrees of freedom, % can be disconnected, 
in the KME setting the distance between the covariances of two distributions is controlled by the distance between their respective KMEs (see Lemma~\ref{lem:bound_trace}).\looseness=-1

\begin{theorem}\label{prop:kme_aggregation}
Let $\cH$ be a RKHS with a kernel bounded by $M=1$, $u_0 := 2\log (Bn_1)$, and $\widehat{\omvect}$ be the output of Algorithm~\ref{alg:Qaggregation} for $\widehat{\nu}_k = \muNE_k$ the empirical KMEs, $\Phi_i^{(1)} = \phi_\kappa(Z_i^{(1)})$ the dataset of agent~$1$ injected in $\cH$ and $C_Q^2,C_P > C_0u_0$ for some absolute constant $C_0$. Then, for any set of agents $V$ that includes agent $1$, we have:
\begin{equation}\label{eq:mse_aggreg}
	\e[2]{\MMD^2\paren[1]{ \widehat{\mbp}(\widehat{\omvect}), \mbp_1}} \leq  \brac[2]{\Delta^2_V + \frac{\tr \Sigma_1+2 \Delta_V}{n_V}}
	+ \frac{Cu_0}{\sqrt{n_1}} \paren[3]{\sqrt{\frac{|V|-1}{n_V}}\vee \frac{1}{\sqrt{n_1}}} \paren[2]{ \frac{\tr \Sigma_1}{\sqrt{d^e_1}}\vee\frac{u_0}{\sqrt{n_1} }}\,,
\end{equation}
where $C>0$ is some absolute constant depending on $C_0$,
\begin{equation}\label{eq:def_deltaV_nV}
	\Delta_V = \sup_{k\in V} \MMD(\mbp_1,\mbp_k), \; \text{and} \; n_V = \sum_{k\in V} n_k\,.
\end{equation}
\end{theorem}

\textbf{Discussion.} Theorem~\ref{prop:kme_aggregation} demonstrates the adaptivity of Algorithm~\ref{alg:Qaggregation}: the learned mixture achieves a near-optimal bias-variance trade-off, with the bias determined by the distance $\Delta_V$ of the selected distributions to the target $\mbp_1$, and the variance controlled by the combined sample size $n_V$ of the selected agents. Since the bound~\eqref{eq:mse_aggreg} holds for any subset $V$, it is in particular valid for the optimal set of agents minimizing it. The ratio $|V|/n_V$ of the selected agents also appears in the bound, reflecting the balance between the number of agents and their data size. Intuitively, the method improves estimation whenever enough agents have distributions close to the target ($\Delta_V$ small) and sufficient data to contribute, yielding a mixture reduces the error below the error $\tr \Sigma_1 / n_1$ (Eq.~\ref{eq:mse_kme}) of the naive local estimate.

More generally, the bound is always at least as good as the naive estimator, up to lower-order terms (second part of Eq.~\ref{eq:mse_aggreg}). These terms are effectively of smaller order in high dimension ($d_1^e$ large) and for a sufficiently large number of local data points, typically $u_0 \simeq \log B \leq \sqrt{n_1}$. Otherwise, the penalization may lead the method to only consider the local data. Below, we illustrate this with an example of agent structure and the resulting performance gains.

\begin{example}[Identical agents]
Suppose there a subset $V$ of agents whose distributions are identical to $\mbp_1$ and who each have at least $n_1$ points, then:
\begin{equation*}%\label{eq:example_mmd_bound}
	\e[2]{ \MMD^2\paren[1]{ \widehat{\mbp}(\widehat{\omvect}), \mbp_1}}
	\leq  \frac{\tr \Sigma_1}{n_V}
	+ \frac{Cu_0}{n_1}  \paren[2]{ \frac{\tr \Sigma_1}{\sqrt{d^e_1}}\vee\frac{u_0}{\sqrt{n_1} }}.
\end{equation*}
In this case, the aggregation achieves performance comparable to an oracle that directly selects these agents. The improvement is capped by a factor $\sqrt{\min(d_1^e,n_1)}$. Conversely, if $V = \{1\}$, i.e. no other agents have the same distribution, the first term dominates, recovering the naive estimation error of order $\tr\Sigma_1/n_1$. %up to smaller order additional terms.
\end{example}

\subsection{Controlling the Excess Risk of the Estimator}\label{key}
As outlined above, learning the mixture weights is only the first step of our approach. The second step is to learn the model $\widehat{\theta} \in \arg\min \sum\widehat{\omega}_k \widehat{\cR}_k(\theta)$ that minimizes the weighted empirical risk. 
Corollary~\ref{thm:kernel} provides excess risk guarantees for this procedure, assuming $\widehat{\theta}$ is an exact minimizer (i.e. ignoring optimization error).

\begin{corollary}\label{thm:kernel}
Under Assumption~\ref{ass:lossinrkhs}, for $\widehat{\theta} \in \arg\min_\theta  \widehat{\cR}_{\widehat{\omvect}}(\theta)$:
\begin{equation}\label{eq:excessrisk_kernel}
	\e[1]{\cR^{(1)}_{\widehat{\theta}}} - \cR^{(1)}_{\theta^*} \leq 2R_\Theta  \e[1]{\MMD\paren[1]{\wh{\mbp}(\widehat{\omvect}), \mbp_1}},
\end{equation}
where $R_\Theta = \sup_{\theta \in \Theta} \norm{h_\theta}_\cH$. In particular, if $\wh{\omvect}$ are the weights defined in Theorem~\ref{prop:kme_aggregation} for the RKHS $\cH$, the right side of \eqref{eq:excessrisk_kernel} is upper bounded by the square root of \eqref{eq:mse_aggreg}.
\end{corollary}

This result shows that minimizing the weighted risk with Q-aggregated weights directly translates into control over the target agent's excess risk, linking the statistical benefit of collaboration to the MMD distance between the aggregated and target distributions.

\section{Practical Federated Algorithm}
\label{sec:rff}

In the previous section, we presented a general algorithm for learning the mixture weights of aggregated empirical risks and derived an excess risk bound for its minimizer. However, a closer look at Algorithm~\ref{alg:Qaggregation} and its practical implementation for RKHS and KMEs shows that, for general kernels, all pairwise distances $\kappa(Z_i^{(1)},Z_j^{(k)})$ between agent~$1$’s data and agent~$k$’s data must be computed (details in Appendix~\ref{APPsec:qaggreg_details}). In practice, this would require centralizing the data (e.g., sharing it with agent~$1$), which violates the core principles of federated learning. Fortunately, for certain kernels, this is not necessary, as illustrated by the following example.

\begin{example}[Linear regression, continued]\label{ex:linear_regression2}
Building on Example~\ref{ex:linear_regression_rkhs}, for a second-order polynomial kernel,
the empirical KMEs can be transmitted directly, since they depend only on the local empirical mean $\bar{Z}_k =n_k^{-1}\sum_{i=1}^{n_k} Z_i^{(k)}$ and uncentered covariance $\bar{C}_k = n_k^{-1} \sum_{i=1}^{n_k} Z_i^{(k)}(Z_i^{(k)})^T$. Specifically, for $z\in \mbr^{d+1}$:
\begin{equation}
	\widehat{\mu}_k(z) = 1+ 2\inner{z,\bar{Z}_k } + z^T \bar{C}_kz.
\end{equation}
The inner products between KMEs required by Algorithm~\ref{alg:Qaggregation} (for empirical error and penalization terms) can then be computed directly from these quantities:
\begin{equation}
	\inner{ \widehat{\mu}_k, \widehat{\mu}_\ell}_\cH = 1 + 2 \inner{ \bar{Z}_k, \bar{Z}_\ell} + \tr{\bar{C}_k\bar{C}_\ell}\,.
\end{equation}
\end{example}

Unfortunately, similar simplifications are not possible for most popular kernels, such as Gaussian or Laplace. To address this, we next present a practical algorithm based on random Fourier features.

\subsection{Random Fourier Features Approximation}\label{sec:RFF_method}

As recalled in Section~\ref{sec:RFF}, Random Fourier Features (RFF) provide a finite-dimensional approximation of the RKHS $\cH$ and associated KMEs.
Using shared coefficients $\Gamma$, each agent can compute its approximated KME $\widehat{\mu}^\Gamma_k$ locally and transmit it to the server. These KMEs are represented as vectors in $\mathbb{R}^D$, which can then be used in Algorithm~\ref{alg:Qaggregation} to estimate some weights $\widehat{\omvect}^\Gamma$. The procedure is given below.

\begin{algorithm}[h]
\caption{Federated Q-aggregation with RFF}
\label{alg:rff}

\begin{algorithmic}
	\State {\bfseries Input:} Distribution $p$ associated to kernel $\kappa$, $D\in \mbn$, local datasets $Z_\bullet$, $C_Q,C_P>0$.
	\State {\bf Sampling:} central server samples RFF coefficients: \\
	$w_s \sim p, b_s \sim \cU([0,2\pi])$ for $s \in \intr{D}$.
	\State  {\bf Sharing RFF coefficents:} central server shares $\Gamma = (w_s,b_s)_{s=1}^D$ to all agents.
	\State  {\bf Local KME:} each agent $k$ compute $\muNE_k^\Gamma = \frac{1}{n_k}\sum_{i=1}^{n_k} \Phi_i^{(k)}$
	with $\Phi_i^{(k)} = \paren[1]{\cos\paren[0]{\inner[0]{w_s, Z^{(k)}_i} + b_s}}_{s=1}^D\in \mbr^D$.
	\State {\bf Sharing KMEs:} $(\muNE^\Gamma_k)_{k=2}^B$ are transmitted to agent $1$.
	\State {\bf Computing weights $\widehat{\omvect}^\Gamma$:} Agent $1$ runs Algorithm~\ref{alg:Qaggregation} with $\widehat{\nu}_k = \muNE^\Gamma_k$, $(\Phi_i^{(1)})_{i=1}^{n_1}$, $C_Q,C_P$ and $M=\sqrt{2}$.
	\State  {\bf Minimizing weighted risk:} run \texttt{FedAvg} to obtain \\
	$\widehat{\theta} \in \argmin_{\theta \in \Theta}\sum_{k=1}^{B} \widehat{\omega}_k^\Gamma \widehat{\cR}_k(\theta)$
	\State {\bf Output:} $\widehat{\theta}$
\end{algorithmic}
\end{algorithm}
\vspace{-0.5em}
Note that the distribution $p$ is fully determined by the kernel $\kappa$; for instance, $p$ is Gaussian for the Gaussian kernel (see \citealp{rahimi2007random} for additional examples). By construction, the RFFs are always bounded by a constant, here $M = \sqrt{2}$. 
Finally, the optimization of the weighted risk in Algorithm~\ref{alg:rff} is performed using \texttt{FedAvg}, but any other federated optimization method could be used instead.

Theorem~\ref{cor:kmerff} below gives a bound on the excess risk of the model learned by Algorithm~\ref{alg:rff}. Similarly to Corollary~\ref{thm:kernel}, it comes from a control of the MMD distance between the resulting empirical measure and the target one.

\begin{theorem}\label{cor:kmerff}
Let $\kappa$ be a translation-invariant kernel bounded by $1$. Let $u_0=\log Bn_1$, and $\hat{\theta}_{\text{RFF}}$ the output of Algorithm~\ref{alg:rff} applied with $C_Q^2,C_P > C_0u_0$ for some absolute constant $C_0$. Then, under Assumption~\ref{ass:lossinrkhs}, for any set $V$ of agents (which includes target agent $1$), we have:	
\begin{equation}\label{eq:excessrisk_kernel2}
	\e[1]{\cR^{(1)}_{\hat{\theta}_{\text{RFF}}}}\! - \!\cR^{(1)}_{\theta^*} \leq 2R_\Theta  \sqrt{\e{\MMD^2\paren[1]{\wh{\mbp}(\wh{\omvect}^\Gamma), \mbp_1}}},
\end{equation}
where $R_\Theta = \sup_{\theta \in \Theta} \norm{h_\theta}_\cH$  and for any set $V$ with $1\in V$:
\begin{multline}
		\e[2]{ \MMD^2\paren[1]{ \wh{\mbp}(\wh{\omvect}^\Gamma) ,\mbp_1}} \leq \brac{\Delta_V^2 + \frac{\tr \Sigma_1}{n_{V}} + \frac{2\Delta_V}{n_V}} \notag\\
	+ \frac{Cu_0}{\sqrt{n_1}} \paren[2]{\sqrt{\frac{|V|-1}{n_V}}\!\vee\! \frac{1}{\sqrt{n_1}}} \paren[2]{ \frac{\tr \Sigma_1}{\sqrt{d^e_1}}\!\vee\!\sqrt{\frac{d_1^e}{D}} \!\vee\! \frac{u_0}{\sqrt{n_1} }} +  C \sqrt{\frac{\log B}{D}}
\end{multline}
where $\Sigma_1$ is the covariance operator of $\mbp_1$ in $\cH$, $d_1^e$ its effective dimension and $\Delta_V$, $n_V$ are defined in \eqref{eq:def_deltaV_nV}.
\end{theorem}

Compared to Theorem~\ref{prop:kme_aggregation} and Corollary~\ref{thm:kernel}, the rate derived here contains additional error terms of order $\mathcal{O}(D^{-\frac{1}{2}})$, which arise from the RFF approximation. The precision of the approximation of the KMEs improves as $D$ increases, and when $D \to \infty$, we recover the result \eqref{eq:mse_aggreg}.
However, larger $D$ increases the dimension of vectors that must be shared, thereby increasing the communication costs. In practice, the choice of $D$ may depend on the federated optimization algorithm used in the last step. 
Notably, standard \texttt{FedAvg} already requires iterative communication of gradients at each round, with dimension equal to the parameter space $\Theta$, so setting $D$ on this scale does not significantly increase communication.
Another advantage of our approach is that KMEs $\muNE_k^\Gamma$ need to be shared only once, allowing all agents to learn their weights locally in parallel.

\subsection{Choice of Kernel}

Our method has so far been presented in a general setting where the kernel $\kappa$ is defined on an arbitrary data domain $\cZ$. The functions in the RKHS associated with a universal kernel \citep{sriperumbudur2008injective,sriperumbudur2010hilbert,sriperumbudur2011universality}, such as the Gaussian or Laplace kernel, can approximate any bounded continuous function, which allows us to consider Assumption~\ref{ass:lossinrkhs} as being approximately satisfied in general (see discussion in Appendix~\ref{APPsec:approx_error}).

Still, the choice of the kernel and the space on which it operates should be adapted to the application context, the data and the type of heterogeneity, as it determines the information transmitted and compared between agents. In unsupervised learning, many tasks, such as density estimation (maximum likelihood), clustering or dimensionality reduction, can be cast as population risk minimization problems like~\eqref{eq:def_gen_error}. In such cases, our framework is applicable regardless of the heterogeneity, and using a universal kernel is generally a good option. In supervised learning settings, however, where $Z=(X,Y)$, treating $Z$ as a single random vector, as done in the previous sections, rather than as a pair $(X,Y)$ with different dimensions, may lead to failures or suboptimal performance depending on the type of heterogeneity.

For the case of concept shift, where the conditional distributions $Y| X$ vary across agents, it is important to use a kernel that places sufficient weight on the $Y$ component, especially when $X$ is high-dimensional. 
In our experiments, we use a weighted Gaussian kernel that emphasizes the $y$ coordinate of $z$, defined as $\kappa_A(z,z') = \exp\paren{- z^T A z'}$ where $z=(x,y) \in \mbr^d \times \mbr$ and $A = \sigma_XI_d \oplus \sigma_YI_1$. This approach can naturally be generalized to incorporate different forms of a priori knowledge about the heterogeneity among agents.\looseness=-1

In the case of covariate shift, the conditional distribution of $Y|X$ is shared across agents, but the marginal distribution of the features $X$ may vary. Here, we propose to learn the weights from the aggregation of KMEs of the features $X$ alone, rather than the full tuple $(X,Y)$, as the heterogeneity arises from the features. The kernel is defined over $\cX$, and so are the RFF. Theoretical guarantees specific to this setting are provided in Appendix~\ref{APPsec:covshift}. 

\section{Experiments}

\begin{figure}[t]  % [t] pour "top", [h] = "here", [b] = "bottom"
	\centering	
	\includegraphics[width=0.7\linewidth]{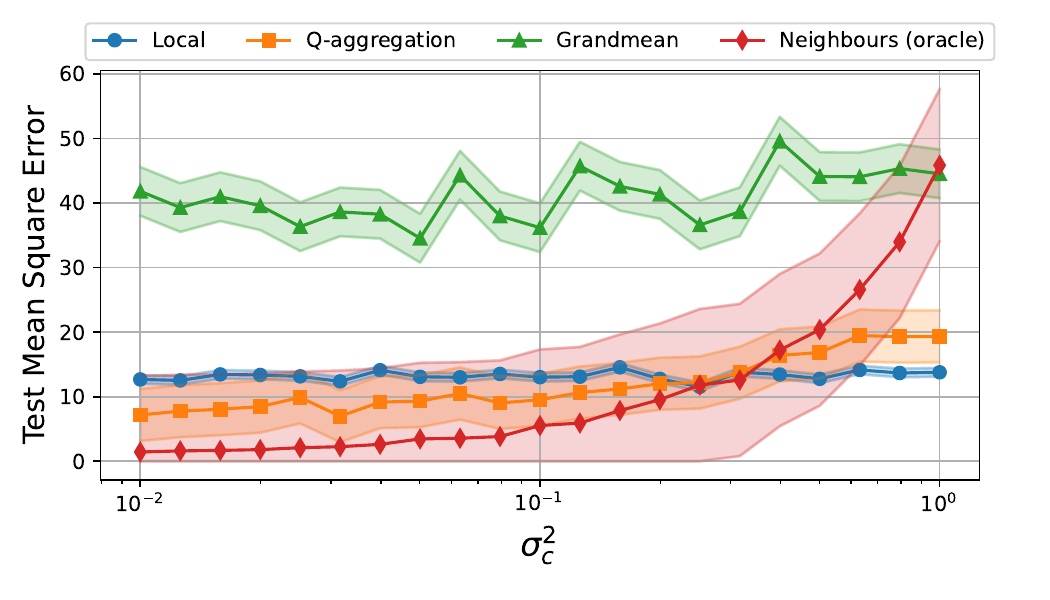}   
	\vspace{-1em}
	\caption{Mean Squared Error and its standard deviation of different approaches in function of the intra-group noise $\sigma^2_c$.
	}	\label{fig:mse_vs_sigma_concept}
	\vspace{-1em}
\end{figure}

We consider different FL settings with heterogeneous agents. We compare Algorithm~\ref{alg:rff} to \textit{Local} training, i.e. using only local data, and to \textit{GrandMean}, where a single model minimizes the weighted risk $\wh{\cR}_{\omvect^{\texttt{gm}}}$ where $(\omega^{\texttt{gm}})_k = \frac{n_k}{\sum_{\ell=1}^{B} n_\ell}$. In some experiments, we further define a notion of \textit{Oracle} having a priori information on agents' similarity. 

For all experiments, we use the RFF approximation of a Gaussian kernel with $D=500$ for synthetic experiments and $D=1000$ for the \texttt{Femnist} dataset.  We emphasize that we \textit{do not} tune the hyperparameters $C_Q$ and $C_P$ of our mixture weight learning approach: they are fixed according to the theory (see experimental details in Appendix~\ref{APPsec:tech_details}).
This illustrates the robustness of our approach.

\subsection{Synthetic Concept Shift}\label{ssec:concept_shift}

We consider a concept shift setting in linear regression. The feature distributions of the agents are identical, with $X_i^{(k)} \sim \cN\paren[1]{(1,\ldots,1),I_d}$, whereas the output distributions vary across agents. Each agent belongs to a group $I_k \sim \cU(\{-1,1\})$, with an intra-group proximity determined by a parameter~$\sigmac^2$. 
For an agent $k \in \intr{B}$:
\begin{gather}
Y_i^{(k)} \sim \inner[1]{\beta_k,X_i^{(k)}} + \cN(0,\sigma^2_Y)\,, \\
\beta_k \sim I_k\sqrt{1- \sigmac^2}\beta_0 + \sigmac \varepsilon_k,\,  \;\beta_0,\varepsilon_k \sim \cN(0,I_d). \label{eq:betak_expe}
\end{gather}
We consider $B=100$ agents having $n_k=10$ points each.
As $\sigma_c$ increases, agents belonging to the same group become increasingly dispersed: when $\sigma_c = 0$, the parameter $\beta_k$ is identical within each cluster, whereas for $\sigma_c = 1$, the parameters $\beta_k$ are completely independent. The parametrization in Eq.~\eqref{eq:betak_expe} ensures that $\e[1]{\norm{\beta_k}^2} = d$ for all values of $\sigma_c$. Thus, locally, the intrinsic difficulty of the problem remains unchanged. The oracle method learns with the agents of the same group by assigning uniform weights like $\omvect^{\texttt{gm}}$, but without taking $\sigma_c$ into consideration.

In Figure~\ref{fig:mse_vs_sigma_concept}, we plot the mean squared error (evaluated using $N_M = 1000$ test points and averaged over $N_r=100$ different training datasets)
of the four methods, as a function of $\sigma^2_c$. We observe a transition at $\sigma_c^2 = 0.5$. For lower values, agents within the same group are relatively close, and leveraging group-specific data improves learning. Beyond this point, the intra-group variability becomes too large, and collaboration leads to a degradation of performance, as illustrated by the results of the oracle method. Our method correctly captures this behavior, adapts to the level of heterogeneity, improves the learning when it is possible and reduces collaboration when $\sigmac^2$ becomes too large.

\subsection{Synthetic Covariate Shift}\label{ssec:cov_shift}

\begin{figure}[t]
\centering
\includegraphics[width=0.55\linewidth]{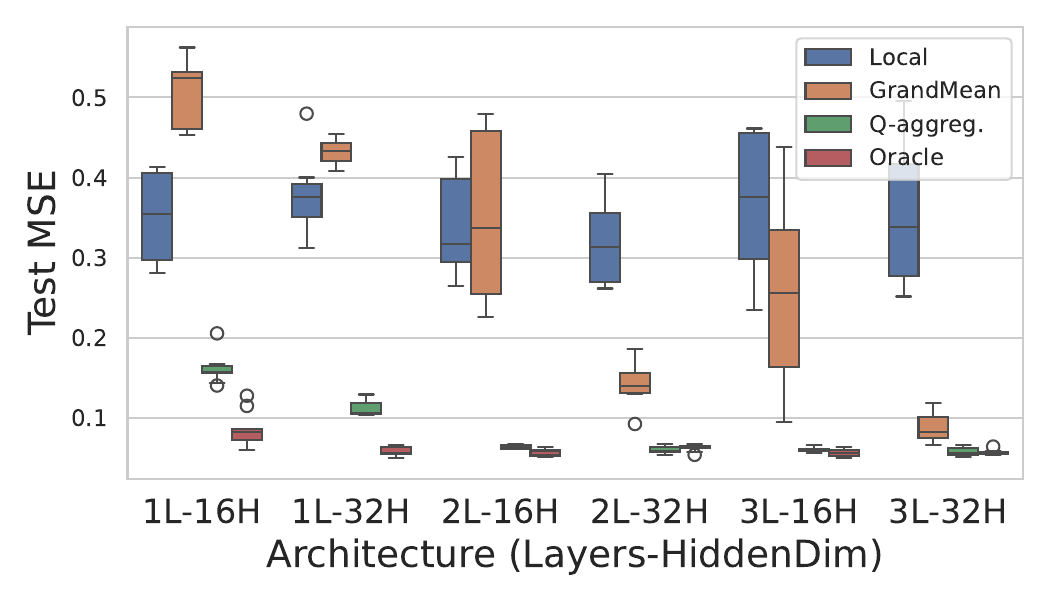}%
\includegraphics[width=0.3\linewidth]{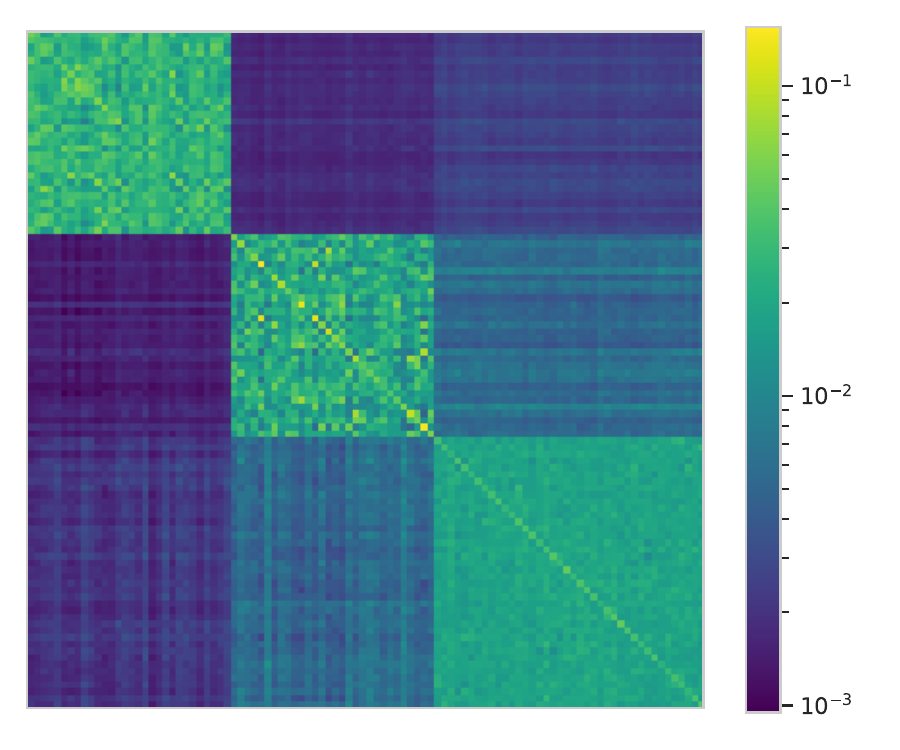}
\caption{Synthetic concept shift. Left side: test MSE in function of the architecture (lower is better). Right size: learned weights.}
\label{fig:mse_covshift}
\vspace{-1em}
\end{figure}

We consider a synthetic data generation setup modeling a covariate shift. The output variable is defined as
\begin{equation*}
Y = \sin(3 X_1) + 0.5 X_2^2 + 0.1 \sum_{i\geq 3} X_i + \cN(0, 0.04),
\end{equation*}
where $X \in \mathbb{R}^d$ is the feature vector. The agents comes from three groups

with group sizes $K_1=K_2=30$ and $K_3=40$ for a total of $B=100$ agents with $n_k=20$. The exact distribution is detailed in Eq.\eqref{eq:distrib_covshift}.
Our regressor is a ReLU neural network with increasing architecture complexity. Figure~\ref{fig:mse_covshift} reports the MSE of the different learning methods across architectures for an agent of the first group. % ($k \leq K_1$).
In this setting, the oracle is trained exclusively on this first group.
Under a covariate shift scenario, one might initially expect that training on all data (\textit{GrandMean}) would be optimal. However, for small models, the network cannot capture the behavior of each subpopulation, making local learning preferable. Our method, in contrast, can identify similar agents and leverage their information effectively. As model capacity increases, learning over all agents becomes increasingly beneficial, and for the largest models, performance approaches that of the oracle.
Figure~\ref{fig:mse_covshift} also displays the weight matrix between agents learned by our approach, where the three clusters are well recovered. Lastly, our method achieves performance very close to that of the oracle.

\subsection{\texttt{FEMNIST} Dataset}\label{ssec:femnist}

We evaluate our approach on the \texttt{FEMNIST} dataset \citep{caldas2018leaf}, a federated variant of MNIST. Each agent holds handwritten character data (both digits and letters, $|\mathcal{Y}| = 62$) exhibiting different writing styles, and our goal is to train a separate classification model for each agent. Some agents share similar handwriting styles, which can be leveraged to improve learning. We view this setting as a covariate shift problem and use RFFs of the isotropic Gaussian kernel only on the features.

As shown in Fig.~\ref{fig:femnist}, our method consistently improves over the GrandMean approach, which is itself generally superior to local training, but may fail for some specific agents. At the opposite, our method has always better always than local training. 
A possible reason explaining why GrandMean performs well could be the low level of heterogeneity in characters' writing.\looseness=-1

\begin{figure}[t]
\centering
\includegraphics[width=0.8\linewidth]{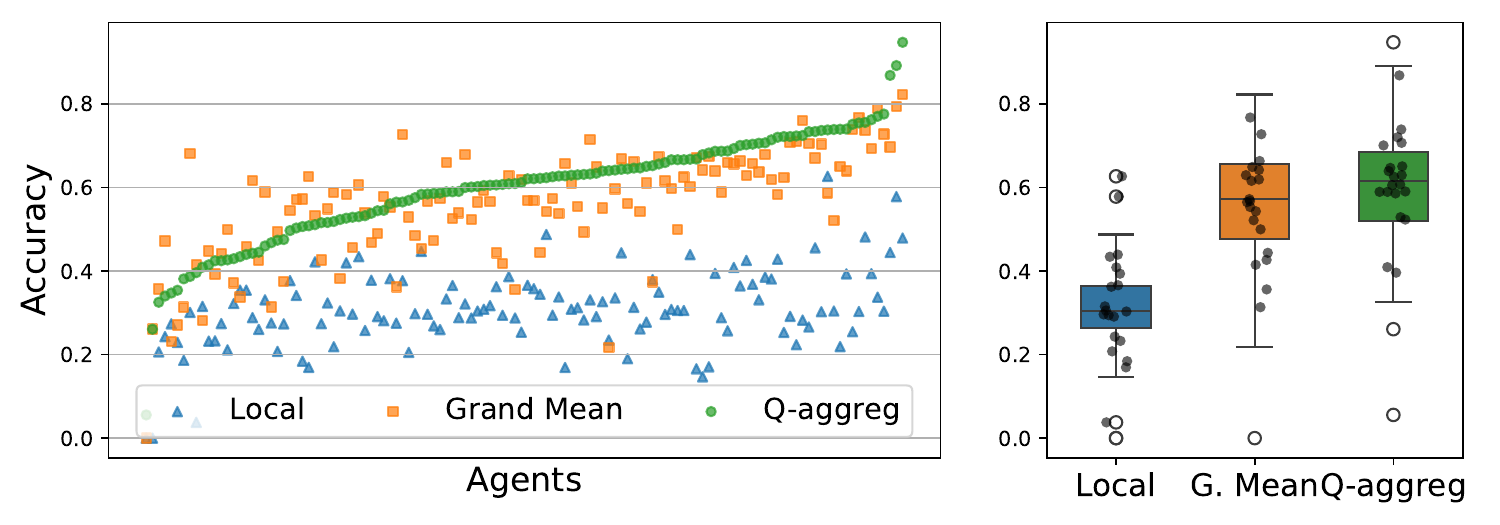}%
\caption{\texttt{FEMNIST}. Accuracy of each agent for each method sorted in function of the Q-aggregation ones and a boxplot of these accuracies other the agents (higher is better)}
\label{fig:femnist}
\vspace{-1em}
\end{figure}

\section{Conclusion}
This work introduces an adaptive algorithm for PFL with strong theoretical guarantees. By assuming that the target loss lies in an RKHS, we reformulate the problem as a high-dimensional mean (KME) estimation task, bridging two research areas and enabling the derivation of rigorous statistical results. The use of random Fourier features approximations further allows a controlled trade-off between communication costs and statistical efficiency in FL settings. This yields a principled mechanism to aggregate heterogeneous client information while adapting to a target client, with both theoretical bounds and empirical results supporting the method's robustness, effectiveness, and interpretability. \looseness=-1

Several key questions remain open. In particular, the choice of kernel is central and deserves further investigation. It is also important to quantify the privacy loss associated with sharing KMEs, and to better understand how general loss functions can be approximated within an RKHS. 
Beyond the setting considered in this work, a natural direction for future research is to extend our framework to scenarios where agents learn their models simultaneously by iteratively aggregating gradients rather than risk measures, enabling even greater adaptability.

\section*{Acknowledgments}
This research was supported in part by the Groupe La Poste, sponsor of the Inria Foundation, in the framework of the FedMalin Inria Challenge.

%\printbibliography
\bibliography{biblio_fed}

\begin{thebibliography}{65}
\providecommand{\natexlab}[1]{#1}
\providecommand{\url}[1]{\texttt{#1}}
\expandafter\ifx\csname urlstyle\endcsname\relax
  \providecommand{\doi}[1]{doi: #1}\else
  \providecommand{\doi}{doi: \begingroup \urlstyle{rm}\Url}\fi

\bibitem[Achituve et~al.(2021)Achituve, Shamsian, Navon, Chechik, and
  Fetaya]{achituve2021personalized}
Idan Achituve, Aviv Shamsian, Aviv Navon, Gal Chechik, and Ethan Fetaya.
\newblock Personalized federated learning with gaussian processes.
\newblock \emph{Advances in Neural Information Processing Systems},
  34:\penalty0 8392--8406, 2021.

\bibitem[Arivazhagan et~al.(2019)Arivazhagan, Aggarwal, Singh, and
  Choudhary]{DBLP:journals/corr/abs-1912-00818}
Manoj~Ghuhan Arivazhagan, Vinay Aggarwal, Aaditya~Kumar Singh, and Sunav
  Choudhary.
\newblock Federated learning with personalization layers.
\newblock \emph{CoRR}, abs/1912.00818, 2019.

\bibitem[Aronszajn(1950)]{aronszajn1950theory}
Nachman Aronszajn.
\newblock Theory of reproducing kernels.
\newblock \emph{Transactions of the American mathematical society}, 68\penalty0
  (3):\penalty0 337--404, 1950.

\bibitem[Bach(2024)]{bach2024learning}
Francis Bach.
\newblock \emph{Learning theory from first principles}.
\newblock MIT press, 2024.

\bibitem[Baraud(2002)]{baraud2002non}
Yannick Baraud.
\newblock {Non-asymptotic minimax rates of testing in signal detection}.
\newblock \emph{Bernoulli}, 8\penalty0 (5):\penalty0 577 -- 606, 2002.

\bibitem[Blanchard and Fermanian(2023)]{blanchard2019nonasymptotic}
Gilles Blanchard and Jean-Baptiste Fermanian.
\newblock Nonasymptotic one-and two-sample tests in high dimension with unknown
  covariance structure.
\newblock In \emph{Foundations of modern statistics}, pages 121--162. Springer,
  2023.

\bibitem[Blanchard et~al.(2018)Blanchard, Carpentier, and
  Gutzeit]{blanchard2018minimax}
Gilles Blanchard, Alexandra Carpentier, and Maurilio Gutzeit.
\newblock {Minimax Euclidean separation rates for testing convex hypotheses in
  $\mathbb{R}^{d}$}.
\newblock \emph{Electronic Journal of Statistics}, 12\penalty0 (2):\penalty0
  3713 -- 3735, 2018.
\newblock \doi{10.1214/18-EJS1472}.

\bibitem[Blanchard et~al.(2024)Blanchard, Fermanian, and
  Marienwald]{blanchard2024estimation}
Gilles Blanchard, Jean-Baptiste Fermanian, and Hannah Marienwald.
\newblock Estimation of multiple mean vectors in high dimension.
\newblock \emph{arXiv preprint arXiv:2403.15038}, 2024.

\bibitem[Bochner(2005)]{bochner2005harmonic}
Salomon Bochner.
\newblock \emph{Harmonic analysis and the theory of probability}.
\newblock Courier Corporation, 2005.

\bibitem[Caldas et~al.(2018)Caldas, Duddu, Wu, Li, Kone{\v{c}}n{\`y}, McMahan,
  Smith, and Talwalkar]{caldas2018leaf}
Sebastian Caldas, Sai Meher~Karthik Duddu, Peter Wu, Tian Li, Jakub
  Kone{\v{c}}n{\`y}, H~Brendan McMahan, Virginia Smith, and Ameet Talwalkar.
\newblock Leaf: A benchmark for federated settings.
\newblock \emph{arXiv preprint arXiv:1812.01097}, 2018.

\bibitem[Caruana(1997)]{caruana1997multitask}
Rich Caruana.
\newblock Multitask learning.
\newblock \emph{Machine learning}, 28\penalty0 (1):\penalty0 41--75, 1997.

\bibitem[Chen et~al.(2018)Chen, Luo, Dong, Li, and He]{chen2018federated}
Fei Chen, Mi~Luo, Zhenhua Dong, Zhenguo Li, and Xiuqiang He.
\newblock Federated meta-learning with fast convergence and efficient
  communication.
\newblock \emph{arXiv preprint arXiv:1802.07876}, 2018.

\bibitem[Chen et~al.(2022)Chen, Xiao, and Pang]{chen2022satellite}
Hao Chen, Ming Xiao, and Zhibo Pang.
\newblock Satellite-based computing networks with federated learning.
\newblock \emph{IEEE Wireless Communications}, 29\penalty0 (1):\penalty0
  78--84, 2022.

\bibitem[Cui et~al.(2022)Cui, Liang, Pan, Chen, Zhang, and
  Wang]{cui2022collaboration}
Sen Cui, Jian Liang, Weishen Pan, Kun Chen, Changshui Zhang, and Fei Wang.
\newblock Collaboration equilibrium in federated learning.
\newblock In \emph{Proceedings of the 28th ACM SIGKDD conference on knowledge
  discovery and data mining}, pages 241--251, 2022.

\bibitem[Deng et~al.(2020)Deng, Kamani, and Mahdavi]{deng2020adaptive}
Yuyang Deng, Mohammad~Mahdi Kamani, and Mehrdad Mahdavi.
\newblock Adaptive personalized federated learning.
\newblock \emph{arXiv preprint arXiv:2003.13461}, 2020.

\bibitem[Ding and Wang(2022)]{ding2022collaborative}
Shu Ding and Wei Wang.
\newblock Collaborative learning by detecting collaboration partners.
\newblock \emph{Advances in Neural Information Processing Systems},
  35:\penalty0 15629--15641, 2022.

\bibitem[Dinh et~al.(2020)Dinh, Tran, and Nguyen]{t2020personalized}
Canh Dinh, Nguyen Tran, and Josh Nguyen.
\newblock Personalized federated learning with moreau envelopes.
\newblock \emph{Advances in neural information processing systems},
  33:\penalty0 21394--21405, 2020.

\bibitem[Duan and Wang(2023)]{duan2023adaptive}
Yaqi Duan and Kaizheng Wang.
\newblock Adaptive and robust multi-task learning.
\newblock \emph{The Annals of Statistics}, 51\penalty0 (5):\penalty0
  2015--2039, 2023.

\bibitem[Efron and Morris(1972)]{efron1972empirical}
Bradley Efron and Carl Morris.
\newblock Empirical bayes on vector observations: An extension of stein's
  method.
\newblock \emph{Biometrika}, 59\penalty0 (2):\penalty0 335--347, 1972.

\bibitem[Elmahallawy and Luo(2022)]{elmahallawy2022asyncfleo}
Mohamed Elmahallawy and Tie Luo.
\newblock Asyncfleo: Asynchronous federated learning for leo satellite
  constellations with high-altitude platforms.
\newblock In \emph{2022 IEEE International Conference on Big Data (Big Data)},
  pages 5478--5487. IEEE, 2022.

\bibitem[Even et~al.(2022)Even, Massouli{\'e}, and Scaman]{even2022sample}
Mathieu Even, Laurent Massouli{\'e}, and Kevin Scaman.
\newblock On sample optimality in personalized collaborative and federated
  learning.
\newblock \emph{Advances in Neural Information Processing Systems},
  35:\penalty0 212--225, 2022.

\bibitem[Fallah et~al.(2020)Fallah, Mokhtari, and
  Ozdaglar]{fallah2020personalized}
Alireza Fallah, Aryan Mokhtari, and Asuman Ozdaglar.
\newblock Personalized federated learning: A meta-learning approach.
\newblock \emph{arXiv preprint arXiv:2002.07948}, 2020.

\bibitem[Feldman et~al.(2014)Feldman, Gupta, and
  Frigyik]{feldman2014revisiting}
Sergey Feldman, Maya~R Gupta, and Bela~A Frigyik.
\newblock Revisiting stein’s paradox: Multi-task averaging.
\newblock \emph{Journal of Machine Learning Research}, 15:\penalty0 3621--3662,
  2014.

\bibitem[George(1986)]{george1986minimax}
Edward~I George.
\newblock Minimax multiple shrinkage estimation.
\newblock \emph{The Annals of Statistics}, pages 188--205, 1986.

\bibitem[Ghari and Shen(2022)]{m2022personalized}
M~Pouya Ghari and Yanning Shen.
\newblock Personalized online federated learning with multiple kernels.
\newblock \emph{Advances in Neural Information Processing Systems},
  35:\penalty0 33316--33329, 2022.

\bibitem[Ghosh et~al.(2020)Ghosh, Chung, Yin, and
  Ramchandran]{ghosh2020efficient}
Avishek Ghosh, Jichan Chung, Dong Yin, and Kannan Ramchandran.
\newblock An efficient framework for clustered federated learning.
\newblock \emph{Advances in neural information processing systems},
  33:\penalty0 19586--19597, 2020.

\bibitem[Gretton et~al.(2012)Gretton, Borgwardt, Rasch, Sch{\"o}lkopf, and
  Smola]{gretton2012kernel}
Arthur Gretton, Karsten~M Borgwardt, Malte~J Rasch, Bernhard Sch{\"o}lkopf, and
  Alexander Smola.
\newblock A kernel two-sample test.
\newblock \emph{The journal of machine learning research}, 13\penalty0
  (1):\penalty0 723--773, 2012.

\bibitem[Hanzely et~al.(2020)Hanzely, Hanzely, Horv\'{a}th, and
  Richtarik]{NEURIPS2020_187acf79}
Filip Hanzely, Slavom\'{\i}r Hanzely, Samuel Horv\'{a}th, and Peter Richtarik.
\newblock Lower bounds and optimal algorithms for personalized federated
  learning.
\newblock In H.~Larochelle, M.~Ranzato, R.~Hadsell, M.F. Balcan, and H.~Lin,
  editors, \emph{Advances in Neural Information Processing Systems}, volume~33,
  pages 2304--2315. Curran Associates, Inc., 2020.

\bibitem[Hashemi et~al.(2024)Hashemi, He, and Jaggi]{hashemi2024cobo}
Diba Hashemi, Lie He, and Martin Jaggi.
\newblock Cobo: Collaborative learning via bilevel optimization.
\newblock \emph{Advances in Neural Information Processing Systems},
  37:\penalty0 15550--15574, 2024.

\bibitem[Hsu et~al.(2012)Hsu, Kakade, and Zhang]{hsu2012tail}
Daniel Hsu, Sham~M Kakade, and Tong Zhang.
\newblock Tail inequalities for sums of random matrices that depend on the
  intrinsic dimension.
\newblock \emph{Electronic Communications in Probability}, 17:\penalty0 1--13,
  2012.

\bibitem[James et~al.(1961)James, Stein, et~al.]{james1961estimation}
William James, Charles Stein, et~al.
\newblock Estimation with quadratic loss.
\newblock In \emph{Proceedings of the fourth Berkeley symposium on mathematical
  statistics and probability}, volume~1, pages 361--379. University of
  California Press, 1961.

\bibitem[Kairouz et~al.(2021)Kairouz, McMahan, Avent, Bellet, Bennis, Bhagoji,
  Bonawitz, Charles, Cormode, Cummings, et~al.]{kairouz2021advances}
Peter Kairouz, H~Brendan McMahan, Brendan Avent, Aur{\'e}lien Bellet, Mehdi
  Bennis, Arjun~Nitin Bhagoji, Kallista Bonawitz, Zachary Charles, Graham
  Cormode, Rachel Cummings, et~al.
\newblock Advances and open problems in federated learning.
\newblock \emph{Foundations and trends{\textregistered} in machine learning},
  14\penalty0 (1--2):\penalty0 1--210, 2021.

\bibitem[Kharrat et~al.(2025)Kharrat, Canini, and
  Horv{\'a}th]{pmlr-v258-kharrat25a}
Salma Kharrat, Marco Canini, and Samuel Horv{\'a}th.
\newblock Dpfl: Decentralized personalized federated learning.
\newblock In Yingzhen Li, Stephan Mandt, Shipra Agrawal, and Emtiyaz Khan,
  editors, \emph{Proceedings of The 28th International Conference on Artificial
  Intelligence and Statistics}, volume 258 of \emph{Proceedings of Machine
  Learning Research}, pages 5086--5094. PMLR, 03--05 May 2025.

\bibitem[Kivinen and Warmuth(1997)]{kivinen1997exponentiated}
Jyrki Kivinen and Manfred~K Warmuth.
\newblock Exponentiated gradient versus gradient descent for linear predictors.
\newblock \emph{Information and computation}, 132\penalty0 (1):\penalty0 1--63,
  1997.

\bibitem[Koltchinskii and Lounici(2016)]{koltchinskii2016asymptotics}
Vladimir Koltchinskii and Karim Lounici.
\newblock Asymptotics and concentration bounds for bilinear forms of spectral
  projectors of sample covariance.
\newblock In \emph{Annales de l’Institut Henri Poincar{\'e}-Probabilit{\'e}s
  et Statistiques}, volume~52, pages 1976--2013, 2016.

\bibitem[Kulkarni et~al.(2020)Kulkarni, Kulkarni, and Pant]{kulkarni2020survey}
Viraj Kulkarni, Milind Kulkarni, and Aniruddha Pant.
\newblock Survey of personalization techniques for federated learning.
\newblock In \emph{2020 fourth world conference on smart trends in systems,
  security and sustainability (WorldS4)}, pages 794--797. IEEE, 2020.

\bibitem[Li et~al.(2021)Li, Hu, Beirami, and Smith]{li2021ditto}
Tian Li, Shengyuan Hu, Ahmad Beirami, and Virginia Smith.
\newblock Ditto: Fair and robust federated learning through personalization.
\newblock In \emph{International conference on machine learning}, pages
  6357--6368. PMLR, 2021.

\bibitem[Mansour et~al.(2020)Mansour, Mohri, Ro, and Suresh]{mansour2020three}
Yishay Mansour, Mehryar Mohri, Jae Ro, and Ananda~Theertha Suresh.
\newblock Three approaches for personalization with applications to federated
  learning.
\newblock \emph{arXiv preprint arXiv:2002.10619}, 2020.

\bibitem[Marfoq et~al.(2021)Marfoq, Neglia, Bellet, Kameni, and
  Vidal]{marfoq2021federated}
Othmane Marfoq, Giovanni Neglia, Aur{\'e}lien Bellet, Laetitia Kameni, and
  Richard Vidal.
\newblock Federated multi-task learning under a mixture of distributions.
\newblock \emph{Advances in neural information processing systems},
  34:\penalty0 15434--15447, 2021.

\bibitem[Marienwald et~al.(2021)Marienwald, Fermanian, and
  Blanchard]{marienwald2021high}
Hannah Marienwald, Jean-Baptiste Fermanian, and Gilles Blanchard.
\newblock High-dimensional multi-task averaging and application to kernel mean
  embedding.
\newblock In \emph{International Conference on Artificial Intelligence and
  Statistics}, pages 1963--1971. PMLR, 2021.

\bibitem[Mart{\'\i}nez-Rego and Pontil(2013)]{martinez2013multi}
David Mart{\'\i}nez-Rego and Massimiliano Pontil.
\newblock Multi-task averaging via task clustering.
\newblock In \emph{International Workshop on Similarity-Based Pattern
  Recognition}, pages 148--159. Springer, 2013.

\bibitem[Muandet et~al.(2014)Muandet, Fukumizu, Sriperumbudur, Gretton, and
  Sch{\"o}lkopf]{muandet2014kernel}
Krikamol Muandet, Kenji Fukumizu, Bharath Sriperumbudur, Arthur Gretton, and
  Bernhard Sch{\"o}lkopf.
\newblock Kernel mean estimation and stein effect.
\newblock In \emph{International Conference on Machine Learning}, pages 10--18.
  PMLR, 2014.

\bibitem[Nemirovski(2004)]{nemirovski2004prox}
Arkadi Nemirovski.
\newblock Prox-method with rate of convergence o (1/t) for variational
  inequalities with lipschitz continuous monotone operators and smooth
  convex-concave saddle point problems.
\newblock \emph{SIAM Journal on Optimization}, 15\penalty0 (1):\penalty0
  229--251, 2004.

\bibitem[Nguyen et~al.(2022)Nguyen, Pham, Pathirana, Ding, Seneviratne, Lin,
  Dobre, and Hwang]{nguyen2022federated}
Dinh~C Nguyen, Quoc-Viet Pham, Pubudu~N Pathirana, Ming Ding, Aruna
  Seneviratne, Zihuai Lin, Octavia Dobre, and Won-Joo Hwang.
\newblock Federated learning for smart healthcare: A survey.
\newblock \emph{ACM Computing Surveys (Csur)}, 55\penalty0 (3):\penalty0 1--37,
  2022.

\bibitem[Philippenko et~al.(2025)Philippenko, Bars, Scaman, and
  Massouli{\'e}]{philippenko2025adaptive}
Constantin Philippenko, Batiste~Le Bars, Kevin Scaman, and Laurent
  Massouli{\'e}.
\newblock Adaptive collaboration for online personalized distributed learning
  with heterogeneous clients.
\newblock \emph{arXiv preprint arXiv:2507.06844}, 2025.

\bibitem[Rahimi and Recht(2007)]{rahimi2007random}
Ali Rahimi and Benjamin Recht.
\newblock Random features for large-scale kernel machines.
\newblock \emph{Advances in neural information processing systems}, 20, 2007.

\bibitem[Rahimi and Recht(2008{\natexlab{a}})]{rahimi2008uniform}
Ali Rahimi and Benjamin Recht.
\newblock Uniform approximation of functions with random bases.
\newblock In \emph{2008 46th annual allerton conference on communication,
  control, and computing}, pages 555--561. IEEE, 2008{\natexlab{a}}.

\bibitem[Rahimi and Recht(2008{\natexlab{b}})]{rahimi2008weighted}
Ali Rahimi and Benjamin Recht.
\newblock Weighted sums of random kitchen sinks: Replacing minimization with
  randomization in learning.
\newblock \emph{Advances in neural information processing systems}, 21,
  2008{\natexlab{b}}.

\bibitem[Razmi et~al.(2022)Razmi, Matthiesen, Dekorsy, and
  Popovski]{razmi2022board}
Nasrin Razmi, Bho Matthiesen, Armin Dekorsy, and Petar Popovski.
\newblock On-board federated learning for dense leo constellations.
\newblock In \emph{ICC 2022-IEEE International Conference on Communications},
  pages 4715--4720. IEEE, 2022.

\bibitem[Rieke et~al.(2020)Rieke, Hancox, Li, Milletari, Roth, Albarqouni,
  Bakas, Galtier, Landman, Maier-Hein, et~al.]{rieke2020future}
Nicola Rieke, Jonny Hancox, Wenqi Li, Fausto Milletari, Holger~R Roth, Shadi
  Albarqouni, Spyridon Bakas, Mathieu~N Galtier, Bennett~A Landman, Klaus
  Maier-Hein, et~al.
\newblock The future of digital health with federated learning.
\newblock \emph{NPJ digital medicine}, 3\penalty0 (1):\penalty0 119, 2020.

\bibitem[Rosasco et~al.(2010)Rosasco, Belkin, and De~Vito]{rosasco2010learning}
Lorenzo Rosasco, Mikhail Belkin, and Ernesto De~Vito.
\newblock On learning with integral operators.
\newblock \emph{Journal of Machine Learning Research}, 11\penalty0 (2), 2010.

\bibitem[Sattler et~al.(2020)Sattler, M{\"u}ller, and
  Samek]{sattler2020clustered}
Felix Sattler, Klaus-Robert M{\"u}ller, and Wojciech Samek.
\newblock Clustered federated learning: Model-agnostic distributed multitask
  optimization under privacy constraints.
\newblock \emph{IEEE transactions on neural networks and learning systems},
  32\penalty0 (8):\penalty0 3710--3722, 2020.

\bibitem[Scaman et~al.(2024)Scaman, Even, Le~Bars, and
  Massouli{\'e}]{scaman2024minimax}
Kevin Scaman, Mathieu Even, Batiste Le~Bars, and Laurent Massouli{\'e}.
\newblock Minimax excess risk of first-order methods for statistical learning
  with data-dependent oracles.
\newblock In \emph{International Conference on Artificial Intelligence and
  Statistics}, pages 3709--3717. PMLR, 2024.

\bibitem[Smith et~al.(2017)Smith, Chiang, Sanjabi, and
  Talwalkar]{smith2017federated}
Virginia Smith, Chao-Kai Chiang, Maziar Sanjabi, and Ameet~S Talwalkar.
\newblock Federated multi-task learning.
\newblock \emph{Advances in neural information processing systems}, 30, 2017.

\bibitem[Smola et~al.(2007)Smola, Gretton, Song, and
  Sch{\"o}lkopf]{smola2007hilbert}
Alex Smola, Arthur Gretton, Le~Song, and Bernhard Sch{\"o}lkopf.
\newblock A hilbert space embedding for distributions.
\newblock In \emph{International conference on algorithmic learning theory},
  pages 13--31. Springer, 2007.

\bibitem[Sriperumbudur et~al.(2008)Sriperumbudur, Gretton, Fukumizu, Lanckriet,
  and Sch{\"o}lkopf]{sriperumbudur2008injective}
Bharath~K Sriperumbudur, Arthur Gretton, Kenji Fukumizu, Gert Lanckriet, and
  Bernhard Sch{\"o}lkopf.
\newblock Injective hilbert space embeddings of probability measures.
\newblock In \emph{21st annual conference on learning theory (COLT 2008)},
  pages 111--122. Omnipress, 2008.

\bibitem[Sriperumbudur et~al.(2010)Sriperumbudur, Gretton, Fukumizu,
  Sch{\"o}lkopf, and Lanckriet]{sriperumbudur2010hilbert}
Bharath~K Sriperumbudur, Arthur Gretton, Kenji Fukumizu, Bernhard
  Sch{\"o}lkopf, and Gert~RG Lanckriet.
\newblock Hilbert space embeddings and metrics on probability measures.
\newblock \emph{The Journal of Machine Learning Research}, 11:\penalty0
  1517--1561, 2010.

\bibitem[Sriperumbudur et~al.(2011)Sriperumbudur, Fukumizu, and
  Lanckriet]{sriperumbudur2011universality}
Bharath~K Sriperumbudur, Kenji Fukumizu, and Gert~RG Lanckriet.
\newblock Universality, characteristic kernels and rkhs embedding of measures.
\newblock \emph{Journal of Machine Learning Research}, 12\penalty0 (7), 2011.

\bibitem[Stein(1956)]{stein1956inadmissibility}
Charles Stein.
\newblock Inadmissibility of the usual estimator for the mean of a multivariate
  normal distribution.
\newblock In \emph{Proceedings of the third Berkeley symposium on mathematical
  statistics and probability, volume 1: Contributions to the theory of
  statistics}, volume~3, pages 197--207. University of California Press, 1956.

\bibitem[Tan et~al.(2022)Tan, Yu, Cui, and Yang]{tan2022towards}
Alysa~Ziying Tan, Han Yu, Lizhen Cui, and Qiang Yang.
\newblock Towards personalized federated learning.
\newblock \emph{IEEE transactions on neural networks and learning systems},
  34\penalty0 (12):\penalty0 9587--9603, 2022.

\bibitem[Tropp(2015)]{tropp2015introduction}
Joel~A Tropp.
\newblock An introduction to matrix concentration inequalities.
\newblock \emph{Foundations and trends{\textregistered} in machine learning},
  8\penalty0 (1-2):\penalty0 1--230, 2015.

\bibitem[Wu et~al.(2023)Wu, Zhang, Yu, Liu, Gu, Zhou, Chen, and
  Cheng]{wu2023personalized}
Yue Wu, Shuaicheng Zhang, Wenchao Yu, Yanchi Liu, Quanquan Gu, Dawei Zhou,
  Haifeng Chen, and Wei Cheng.
\newblock Personalized federated learning under mixture of distributions.
\newblock In \emph{International Conference on Machine Learning}, pages
  37860--37879. PMLR, 2023.

\bibitem[Xu et~al.(2021)Xu, Glicksberg, Su, Walker, Bian, and
  Wang]{xu2021federated}
Jie Xu, Benjamin~S Glicksberg, Chang Su, Peter Walker, Jiang Bian, and Fei
  Wang.
\newblock Federated learning for healthcare informatics.
\newblock \emph{Journal of healthcare informatics research}, 5\penalty0
  (1):\penalty0 1--19, 2021.

\bibitem[Zantedeschi et~al.(2020)Zantedeschi, Bellet, and
  Tommasi]{zantedeschi2020fully}
Valentina Zantedeschi, Aur{\'e}lien Bellet, and Marc Tommasi.
\newblock Fully decentralized joint learning of personalized models and
  collaboration graphs.
\newblock In \emph{International Conference on Artificial Intelligence and
  Statistics}, pages 864--874. PMLR, 2020.

\bibitem[Zhang and Yang(2021)]{zhang2021survey}
Yu~Zhang and Qiang Yang.
\newblock A survey on multi-task learning.
\newblock \emph{IEEE transactions on knowledge and data engineering},
  34\penalty0 (12):\penalty0 5586--5609, 2021.

\end{thebibliography}
\newpage

\appendix

\section{Practical implementation of Q-aggregation (Algorithm~\ref{alg:Qaggregation})}\label{APPsec:qaggreg_details}
\batodo{Restate the algorithm here ? }

\subsection{Practical computation of the terms}

Algorithm~\ref{alg:Qaggregation} is presented in a very general form and involves, in its computations, the empirical covariance operator $\widehat{\Sigma}_1$, which may appear impractical to compute. To clarify the implementation of the algorithm, we detail below the computation of each term. We first provide closed-form expressions as functions of the local data $(\Phi_i^{(1)})_i$ and the empirical means $\widehat{\nu}_k$, and then consider the case where $\mathcal{H}$ is an RKHS, expressing each quantity in terms of the kernel $\kappa$ and the data from each sample. These equations are rewriting of expressions from \citet{blanchard2024estimation}.

\textbf{Closed forms.} If it is possible to compute the distance between local points and empirical means (finite dimension or specific kernel such as polynomial ones, see Example~\ref{ex:linear_regression2}), the following expression can be used.
\begin{gather}
	\tr \widehat{\Sigma}_1 = \frac{1}{n_1-1} \sum_{i=1}^{n_1} \norm{ \Phi_i^{(1)} - \widehat{\nu}_1}^2_\cH\,,\\
	Q(\omvect) = \frac{1}{\sqrt{n_1}}\sum_{k=1}^{B} \omega_k \sqrt{q_k}\,, \quad \text{where} \quad q_k = \frac{1}{n_1-1} \sum_{i=1}^{n_1} \inner{\Phi_i^{(1)} - \widehat{\nu}_1, \widehat{\nu}_k - \widehat{\nu}_1 }^2_\cH
\end{gather} 

\textbf{General kernel form.} For certain kernels, such as the Gaussian or Laplacian kernel, computing distances between kernel mean embeddings (KMEs) necessitates access to the entire dataset. Consequently, the algorithm is restricted to centralized settings or must rely on random Fourier features, as outlined in Section~\ref{sec:RFF_method}, which effectively reduces the problem to the previously discussed formulations.
\begin{equation}
	\tr \widehat{\Sigma}_1 = \frac{1}{2(n_1-1)} \sum_{i\neq j =1}^{n_1} \paren{ \kappa(Z_i^{(1)}, Z_i^{(1)}) - 2 \kappa(Z_i^{(1)}, Z_j^{(1)}) + \kappa(Z_j^{(1)}, Z_j^{(1)})}\,,
\end{equation}
and
\begin{multline}
	q_k = \frac{1}{n_1-1} \sum_{i=1}^{n_1} \paren[2]{\frac{1}{n_k} \sum_{j=1}^{n_k} \kappa(Z_i^{(1)}, Z_j^{(k)}) - \frac{1}{n_1} \sum_{j=1}^{n_1} \kappa(Z_i^{(1)}, Z_j^{(1)}) }^2 \\
	- \frac{n_1}{n_1-1} \paren[2]{  \frac{1}{n_1n_k} \sum_{i=1}^{n_1} \sum_{j=1}^{n_k} \kappa(Z_i^{(1)}, Z_j^{(k)}) - \frac{1}{n_1^2} \sum_{i=1}^{n_1}\sum_{j=1}^{n_1} \kappa(Z_i^{(1)}, Z_j^{(1)})}^2
\end{multline}

\subsection{Optimization}

%\begin{wrapfigure}{r}{0.5\textwidth} % 'r' = droite, 0.55 largeur de l'algorithme
%	\begin{minipage}{0.5\textwidth} % ajuste la largeur ici
%		\raggedleft
		\begin{algorithm}[H]
			\caption{Exponential gradient descent}
			\label{alg:expo_descent}
			\begin{algorithmic}
				\State {\bf Inputs.} Gradient $\nabla f : \mbr^B \to \mbr^B$, initialization point $\omega_0 \in \Spx_B$,\\
				learning rate $\eta>0$, number of steps~$T$.
				\For{$t=0,\ldots, T-1$}
				\State {\bf Compute gradient.} $g_t = \nabla f(\omega_t)$
				\State {\bf Proxy update.} For $i\in \intr{B}$, 
				$(\omega'_t)_i = (\omega_t)_i e^{-\eta(g_t)_i} \big\slash\paren{\sum_{j=1}^{B} e^{-\eta(g_t)_j}}$
				\State {\bf Compute proxy gradient.} $g_t' = \nabla f(\omega'_t)$
				\State {\bf Update.} For $i\in \intr{B}$, 
				$(\omega_{t+1})_i = (\omega_t)_i e^{-\eta(g'_t)_i} \big\slash\paren{\sum_{j=1}^{B} e^{-\eta(g'_t)_j}}$
				\EndFor
				\State {\bf Output:} $\omega_{T}$. 
			\end{algorithmic}
		\end{algorithm}
	%\end{minipage}
%\end{wrapfigure}

The minimized quantity of Algorithm~\ref{alg:Qaggregation} can be expressed as a quadratic form. Indeed for any $\omvect \in \Spx_B$:
\begin{align*}
	\widehat{L}_1(\omvect) + C_Q\wh{Q}_1(\omvect) + C_P\wh{P}_1(\omvect) = \omvect^TA\omvect + \inner{ \omvect, b}\,,
\end{align*}
where $A = \paren{ \inner{ \widehat{\nu}_k - \wh{\nu}_1,\widehat{\nu}_\ell - \wh{\nu}_1}_\cH}_{k,\ell=1}^T$ and $b$ is just the vectorial sum of $\wh{P}$ and $\wh{Q}$ and adding $\tr\wh{\Sigma}_1/n_1$ at the firstcoordinate . We then find the minimum of this quadratic form by an exponential gradient descent \citep{kivinen1997exponentiated}. To a better a convergence, we adjust it using the Prox-Method of \citet{nemirovski2004prox}. The learning rate is chosen as $\eta = c/\paren{ 2\norm{A}_{op}+ \norm{b}_\infty}$, since $2\norm{A}_{op}+ \norm{b}_\infty$ upper bounds the Lipschitz norm of the gradient of $\omvect \mapsto \omvect^TA\omvect + \inner{b,\omvect}$. The parameter $c$ is fixed at $c=0.5$ and the number of gradients step at $T=1000$ in all the experiments.

\section{Technical details of the experiments}\label{APPsec:tech_details}

For experiments involving neural networks, we report the performance of each method corresponding to the best test accuracy achieved during training. This choice avoids the need to tune stopping times for the different algorithms, which is itself a nontrivial issue in the federated learning setting.
\subsection{Synthetic concept shift}

Table~\ref{tab:hyperparams_concept} presents the different parameter used in the experiments of Section~\ref{ssec:concept_shift}. To capture the concept shift, we rescaled the features impact on the kernel.
\begin{table}[h]
	\centering
	\begin{tabular}{l|ll}
		\toprule
		\textbf{Type} & \textbf{Parameter} & \textbf{Value} \\
		\midrule
		& Dimension & $d=20$ \\
		& Noise variance & $\sigma_Y^2=2$ \\
		Data & Number of points per agent & $n_k=10$ \\
		& Number of agents & $B=100$ \\
		& Number of repetitions & $N_r=100$ \\
		& Number of test points for evaluating MSE & $N_M=1000$ \\
		\midrule
		& Dimension of random features & $D=500$ \\
		& Kernel  & $\kappa\paren{(x,y),(x',y')}$ \\
		Method& &$= \exp \paren{ - \norm{x-x'}^2/\sqrt{d+1}- \paren{y-y'}^2}$\\
			& RFF distribution & $p \sim \cN(0,A)$ with $A =
		\begin{bmatrix}
			\frac{I_d}{d+1} & 0 \\
			0 & 1
		\end{bmatrix}$  \\
		
		& Parameter of Q-aggregation & $C_Q^2 = C_P =\log B$ \\
		& Model & Linear regression \\
		\bottomrule
	\end{tabular}
	\caption{Parameters of the synthetic concept shift experiments of Section~\ref{ssec:concept_shift}}
	\label{tab:hyperparams_concept}
\end{table}

\subsection{Synthetic covariate shift}

Table~\ref{tab:hyperparams_covshift} presents the different parameter used in the experiments of Section~\ref{ssec:cov_shift}. 
\begin{table}[h]
	\centering
	\begin{tabular}{l|ll}
		\toprule
		\textbf{Type} & \textbf{Parameter} & \textbf{Value} \\
		\midrule
		& Dimension & $d=4$ \\
		& Intra group variance & $v_1^2=0.01, v_2^2= 0.3$ \\
		& Number of points per agent & $n_k=20$ \\
		Data 	& Number of agents by group & $K_1 = K_2 = 30$ \\
		& Center of group 2 & $\mu_0 = (2,\ldots,2)$\\
		& Number of agents by group & $B=100$ \\
		& Number of repetitions & $N_r=20$ \\
		& Number of test points for evaluating MSE & $N_M=2000$ \\
		\midrule
		& Dimension of random features & $D=500$ \\
		& Kernel  & Gaussian kernel\\
		& RFF distribution & $p \sim \cN(0,I_d)$ \\
		
		Method		& Parameter of Q-aggregation & $C_Q = C_P =1$ \\
		& Model & ReLU neural networks  \\
		&Architectures & Number of hidden layers $\in \set{1,2,3}$,\\
		& & Hidden dimensions $\in \{16,32\}$\\
		&Number of epochs & $n_e = 2000$ \\
		&Learning rate & $lr= 0.001$\\
		\bottomrule
	\end{tabular}
	\caption{Parameters of the synthetic covariance shift experiments of Section~\ref{ssec:cov_shift}}
	\label{tab:hyperparams_covshift}
\end{table}
The distribution of the features is
\begin{equation}\label{eq:distrib_covshift}
	X_i^{(k)}\! \sim 
	\begin{cases}
		\mathcal{N}(\mu_k, \sigma_1^2 I_d),\,  1 \leq k \leq K_1,\, \mu_k \sim \mathcal{N}(0, v_1^2 I_d),\\
		\mathcal{N}(\mu_k, \sigma_2^2 I_d), \, K_1 < k \leq K_2,\, \mu_k \sim \mathcal{N}(\mu_0, v_2^2 I_d),\\
		\mathcal{U}([-6,6]^d),\,  k > K_2.
	\end{cases}
\end{equation}

\subsection{Femnist dataset}
The model used in the experiments of Section~\ref{ssec:femnist} in a ReLU neural network with $1$ hidden layer of dimension $32$. It is trained during $2000$ epochs with a learning rate of $0.001$. We only consider $B=192$ agents of this dataset. The test and train sizes are represented in Figure~\ref{fig:size_femnist}. The RFFs are the Gaussian ones with a dimension $D=1000$. The ambient dimension of the features is $d = 28\times28= 782$. The Q-aggregation is applied with $C_Q^2 = C_P = \log B$.

\begin{figure}[h] 
	\centering
	\includegraphics[width=0.7\linewidth]{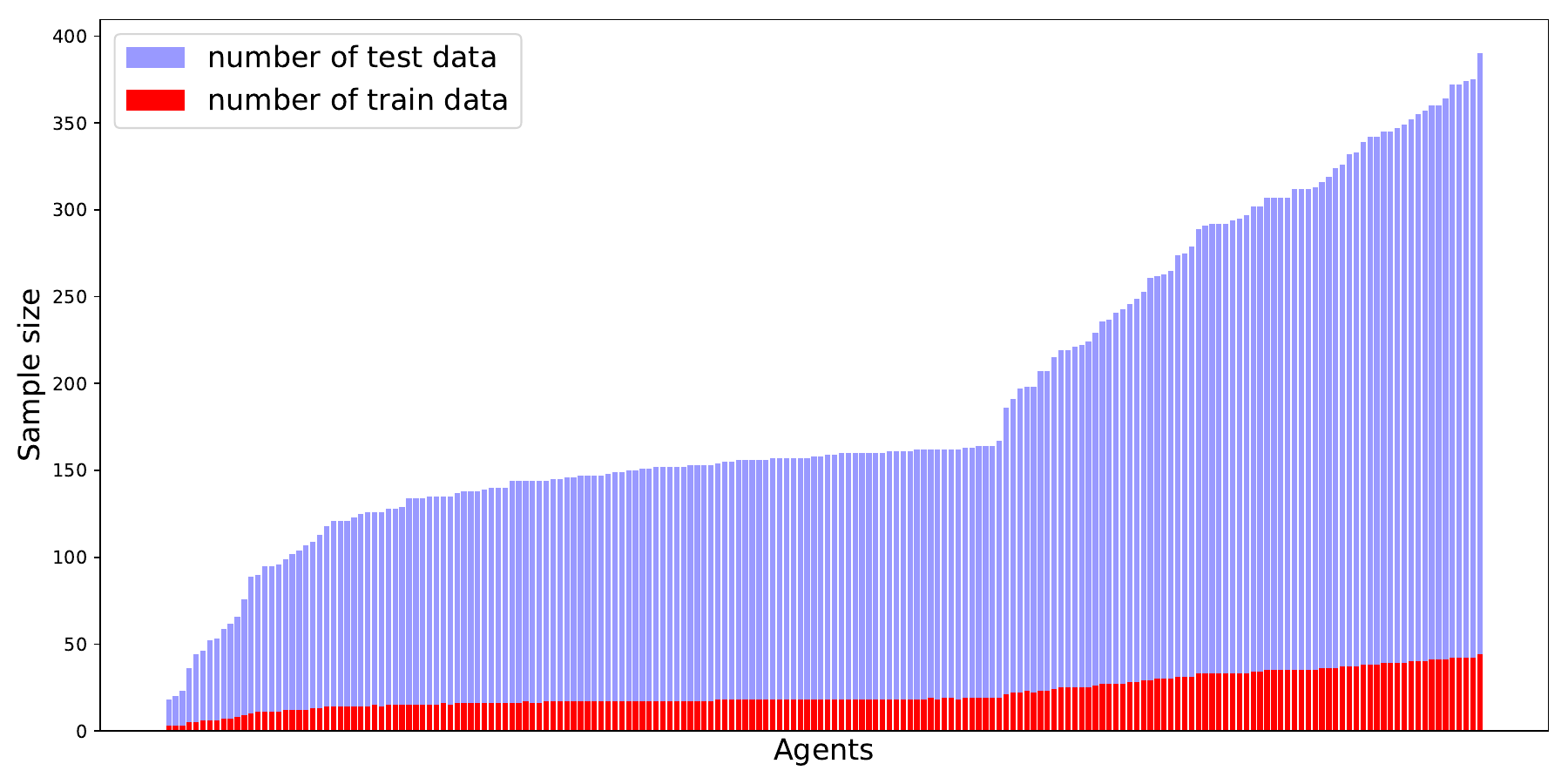}  
	\caption{Number of train and test points for each agent.}
	\label{fig:size_femnist}
\end{figure}

\section{Theoretical result in case of covariate shift}\label{APPsec:covshift}

We present in this section some theoretical results for our approach in case of covariate shift. In this case, we propose to directly learned the aggregation weights from the KME of the features instead of the KMEs of the tuple.
Then, under Assumption~\ref{ass:covshift}, the approximation error of the risk is controlled by the MMD distance between the (weighted) empirical distribution of the features and the one of the target distribution, and does not involve the dependence in $Y$.

\begin{assumption}[Covariate shift]\label{ass:covshift}
	$\cZ = \cX \times \cY$ and for a kernel $\kappa$ on $\cX$, for all $\theta \in \Theta$ and $y\in \cY$, there exists $c_{\theta} \in \Theta$ and $h_{\theta,y} \in \cH_\cX$ such that $\ell_\theta(x,y) = c_{\theta} + h_{\theta,y}(x)$.	
\end{assumption}

\begin{proposition}\label{prop:covshift}
	Let Assumption~\ref{ass:covshift} be satisfied and that the conditional distribution $\mbp^{Y|X}$ is common over the agents, i.e. $\mbp_k = \mbp^{Y|X} \mbp^X_k$ for $k\in \intr{B}$. Then, for any $\widehat{\omvect}$ the weights and $\theta \in \Theta$:
	\begin{equation}\label{eq:covshift}
		\e{ \abs{\widehat{\cR}_{\wh{\omvect}}(\theta) - \cR_1(\theta)}} \leq R_{\theta}^{\cY}\e{ \MMD\paren[1]{\mbp_1^X, \widehat{\mbp}^X(\widehat{\omvect})}}
		%\leq \sup_{y\in \cY} \norm{h_{\theta,y}}_{\cH^X_\kappa} \e{ \MMD_\kappa\paren[1]{\mbp_1^X, \widehat{\mbp}^X(\widehat{\omvect})}}\,,
	\end{equation}
	where $\widehat{\mbp}^X(\omvect) =\ \sum_{k=1}^{B}\omega_k \wh{\mbp}_k^X$ is the empirical mixture of the features and $R_{\theta}^{\cY}= \sup_{y\in \cY} \norm{h_{\theta,y}}_{\cH_X}$.	
\end{proposition}
This result justifies the choice of a kernel just defined on the features, the aggregation would then lead to an optimal aggregation of the features distribution of the agents. However, we emphasize that the guarantee \eqref{eq:covshift} is weaker than those used in the rest of the paper such as \eqref{eq:generror_to_mmd} or \eqref{eq:excessrisk_kernel}, since the control of the risk approximation holds only for a fixed $\theta$ and not uniformly over $\Theta$.

{\bf Proof of Proposition~\ref{prop:covshift}.}
Without loss of generality we can assume $c_\theta = 0$. Let $H_\theta \in \cH_X$ be defined by $H_\theta(x) = \e{ h_{\theta,Y}(x) | X=x}$, then for any $\theta \in \Theta$:
\begin{equation}\label{eq:covshift_aux1}
	\cR_1(\theta) =  \e{ H_\theta(X) } =  \inner{ H_\theta, \mu_{\mbp^X_1}}\,,
\end{equation}
where $\mbp^X_1$ is the distribution of the features of $\mbp_1$. Moreover:
\begin{align}\label{eq:covxhift_aux2}
	\e{\widehat{\cR}_{\wh{\omvect}}(\theta)} = \e[3]{\sum_{k=1}^{B} \widehat{\omega}_k \frac{1}{n_k} \sum_{i=1}^{n_k} \e[2]{h_{\theta, Y_i^{(k)}}(X_i^{(k)}) |X_i^{(k) } } } = \e[3]{\sum_{k=1}^{B} \widehat{\omega}_k \frac{1}{n_k} \sum_{i=1}^{n_k} H_\theta(X_i^{(k)})} =\e[3]{
		\inner[2]{ H_\theta, \sum_{k=1}^B\widehat{\omega}_k \widehat{\mu}_{\mbp^X_k} }_{\cH_X}}\,,
\end{align}
where $\widehat{\mu}_{\mbp^X_k}$ is the KME of the empirical distribution of the features of agent $k$. We have used for the conditioning that the weights $\widehat{\omega}$ are learned from the features $X_i^{(k)}$. Combining \eqref{eq:covshift_aux1} and \eqref{eq:covxhift_aux2} leads to the result:
\begin{equation}
	\abs{ 	\e{\widehat{\cR}_{\wh{\omvect}}(\theta)} - \cR_1(\theta) } \leq \abs[3]{ \e[2]{\inner[2]{ H_\theta, \mu_{\mbp^X_1} - \sum_{k=1}^B\widehat{\omega}_k \widehat{\mu}_{\mbp^X_k} }_{\cH_X}}} \leq R_\theta^\cY \e{ \MMD\paren[1]{\mbp_1^X, \widehat{\mbp}^X(\widehat{\omvect})}}\,,
\end{equation}
using that $\norm{H_\theta}_{\cH_X} \leq \e[2]{ \norm{h_{\theta,Y}}_{\cH_X} | X=x} \leq R_\theta^\cY$, by Jensen's inequality. \qed

\section{Universal kernel and approximation of the loss function}\label{APPsec:approx_error}

An important property of kernels is universality. A kernel defined on a space $\cZ$ is said to be universal if its associated RKHS $\cH$ is dense in the space of continuous functions with respect to the uniform norm. That is, for any continuous function $f$ on $\cZ$ and any $\varepsilon>0$, there exists a function $h \in \cH$ in the RKHS such that its distance to $f$ in the uniform norm is smaller than $\varepsilon$, $\norm{f-h}_\infty < \varepsilon$. Many kernels have been identified for different types of spaces with this property; see for example \citet{muandet2014kernel} for a review.

Our approach relies on the assumption that the loss functions $\ell_\theta$ belong to the RKHS, up to a constant term $c_\theta$ (Assumption~\ref{ass:lossinrkhs}). For a universal kernel, this assumption is expected to hold approximately: for any loss function, there exists a function $h_{\theta,\varepsilon}$ in the RKHS at distance at most $\varepsilon$. However, in order to obtain a bound on the excess risk, we must also control the RKHS norm of such a function, which, to the best of our knowledge, is not generally quantifiable. This issue is formalized in the following lemma.

\begin{lemma}\label{lem:approx}
	Let $\widehat{\theta} \in \argmin_{\theta \in \Theta} \wh{\cR}_{\wh{\omvect}}(\theta)$ for some weights $\wh{\omvect} \in \Spx_B$, then:
	\begin{equation}
		\e{\cR_{1}(\wh{\theta})} - \cR_1(\theta_1^*) \leq \min_{\varepsilon\geq0}\brac{ 2R_{\Theta,\varepsilon} \e{\MMD\paren{\widehat{\mbp}_{\widehat{\omvect}},\mbp_1}} + 4\varepsilon}
	\end{equation}
	where:
	\begin{equation}
		R_{\Theta,\varepsilon} = \sup_{\theta \in\Theta} \min_{\substack{c>0, h\in \cH :\\ \norm{ h+c -\ell_\theta}_\infty < \varepsilon} }\norm{h}_\cH\,.
	\end{equation}
\end{lemma} 
If Assumption~\ref{ass:lossinrkhs} is satisfied, we recover Lemma~\ref{lem:excessrisk_mmd}. As a RKHS contained generally regular functions, the quantity $R_{\Theta,\varepsilon}$ can be interpreted as the complexity of the model relatively to the kernel.

{\bf Proof of Lemma~\ref{lem:approx}.} Let $\varepsilon\geq0$ such that $	R_{\Theta,\varepsilon} $ is finite. Then for any $\ell_\theta$ there exists $h_\theta \in \cH$ and $c_\theta >0$ such that $\norm{ h+c -\ell_\theta}_\infty < \varepsilon$. Then:
\begin{align}
	\e{\cR_{1}(\wh{\theta})} = 	\e{\ell_{\wh{\theta}}(Z)} \leq \varepsilon + c_\theta + \e{h_{\wh{\theta}}(Z)}\,.
\end{align}
As $h_{\wh{\theta}} \in \cH$, then $\e{h_{\wh{\theta}}(Z)} - \sum_{k=1}^{B} \widehat{\omega}_k \frac{1}{n_k}\sum_{i=1}^{n_k}  h_{\wh{\theta}}(Z_i^{(k)}) \leq  R_{\Theta,\varepsilon} \e{\MMD\paren{\widehat{\mbp}_{\widehat{\omvect}},\mbp_1}}$. Using again the proximity with $\ell_\theta$ and the definition of $\widehat{\theta}$, we obtain:
\begin{align*}
	\sum_{k=1}^{B} \widehat{\omega}_k \frac{1}{n_k}\sum_{i=1}^{n_k}  h_{\wh{\theta}}(Z_i^{(k)}) \leq \sum_{k=1}^{B} \widehat{\omega}_k \frac{1}{n_k}\sum_{i=1}^{n_k}  \ell_{\wh{\theta}}(Z_i^{(k)}) - c_\theta + \varepsilon &\leq  \sum_{k=1}^{B} \widehat{\omega}_k \frac{1}{n_k}\sum_{i=1}^{n_k}  \ell_{\theta^*}(Z_i^{(k)}) - c_\theta + \varepsilon \\
	&=\wh{\cR}_{\wh{\omvect}}(\theta ^*) - c_\theta + \varepsilon\,.
\end{align*}
Using the same transformation as at the beginning of the proof, but with $\theta^*$ in place of $\widehat{\theta}$, we obtain the final result.
\qed

\section{Random Fourier Features results}\label{APPsec:RFF}

\subsection{Approximation}
We begin this section by a control in probability between the MMD in the original RKHS and the MMD computed in the random fourier features RKHS.

\begin{lemma}\label{lem:approx_kme_rff}
	Let $\mbp$ and $\mbq$ two distributions on $\cZ$, $\mu_\mbp$ and $\mu_\mbq$ their respective KME in $\cH$ and $\mu_\mbp^\Gamma$ and $\mu_\mbq^\Gamma$ their respective KME in the random Fourier features RKHS $\HFF$. For any $\delta\in (0,1)$:
	\begin{equation}
		\probb[2]{\Gamma}{ \norm{\mu_\mbp - \mu_\mbq }_\cH \leq  \norm{\mu_\mbp^\Gamma - \mu_\mbq^\Gamma }_{\HFF} + C \sqrt{\frac{\log \delta^{-1}}{D}}} \geq 1 - \delta\,,
	\end{equation}
	for some constant $C>0$ depending only on a bound on the kernel.
\end{lemma}

\begin{proof}
	For the rest of the proof $X,X'$ (resp. $Y,Y'$) will denote independent random variables of distribution $\mbp$ (resp.~$\mbq$) and $\phi_\Gamma(z) = (D^{-\nicefrac{1}{2}}\phi(\gamma_i,z))$ will be the random Fourier features map. We recall that:
	\begin{equation}
		\norm{\mu_\mbp - \mu_\mbq }_\cH^2 = \e{ \kappa(X,X')} -2 \e{\kappa(X,Y)} + \e{\kappa(Y,Y')}\,,
	\end{equation}
	and:
	\begin{equation*}
		\norm{\mu_\mbp^\Gamma - \mu_\mbq^\Gamma }_{\HFF}^2 = \frac{1}{D}\sum_{i=1}^{D} \paren[2]{ \e[1]{\phi(\gamma_i,X) | \gamma_i} - \e[1]{\phi(\gamma_i,Y) | \gamma_i}}^2\,.
	\end{equation*}
	Let us first compute the expectation over the random Fourier features $\Gamma$ of the KME distance in $\HFF$. As the features $\gamma_i$ are i.i.d., we get
	\begin{align*}
		\ee{\gamma_1}{\norm{\mu_\mbp^\Gamma - \mu_\mbq^\Gamma }_{\HFF}^2} & = \ee[2]{\Gamma}{ \paren[2]{\e[1]{\phi(\gamma_1,X) | \gamma_1} - \e[1]{\phi(\gamma_1,Y) | \gamma_1}}^2}\\
		&= \ee[2]{\gamma_1}{ \e[1]{\phi(\gamma_1,X)\phi(\gamma_1,X') | \gamma_1} -2 \e[1]{\phi(\gamma_1,X)\phi(\gamma_1,Y) | \gamma_1} + \e[1]{\phi(\gamma_1,Y)\phi(\gamma_1,Y') | \gamma_1}} \\
		&= \ee[2]{\gamma_1}{ \phi(\gamma_1,X)\phi(\gamma_1,X') -2 \phi(\gamma_1,X)\phi(\gamma_1,Y)  + \phi(\gamma_1,Y)\phi(\gamma_1,Y')}\,, 
	\end{align*}
	after developing the square. Using that for any $z,z'\in \cZ$, $\ee{\gamma_1}{ \phi(\gamma_1,z)\phi(\gamma_1,z')} = \kappa(z,z')$, we obtain:
	\begin{equation*}
		\ee{\gamma_1}{\norm{\mu_\mbp^\Gamma - \mu_\mbq^\Gamma }_{\HFF}^2} = \e{ \kappa(X,X')} -2 \e{\kappa(X,Y)} + \e{\kappa(Y,Y')} = \norm{\mu_\mbp - \mu_\mbq }_\cH^2\,.
	\end{equation*}
	Applying Bernstein's inequality to the i.i.d. random variables $Z_i = \paren[2]{ \e[1]{\phi(\gamma_i,X) | \gamma_i} - \e[1]{\phi(\gamma_i,Y) | \gamma_i}}^2 -    \norm{\mu_\mbp - \mu_\mbq }_\cH^2$, we get that, with probability at least $1-\delta$:
	\begin{equation}\label{eq:aux1_RFFapprox}
		\norm{\mu_\mbp - \mu_\mbq }_\cH^2 \leq  \norm{\mu_\mbp^\Gamma - \mu_\mbq^\Gamma }_{\HFF}^2 + C \sqrt{\frac{\log \delta^{-1}}{D}}\norm{\mu_\mbp - \mu_\mbq }_\cH + C\frac{\log \delta^{-1}}{D}\,,
	\end{equation}
	for some absolute constant $C$. We have used that the random variables $Z_i$ are upper bounded by a constant as $\phi$ and the kernel $\kappa$ are supposed bounded and that $\var{Z_i} \leq \e{Z_i^2} \leq C \e{Z_i}$. After inverting \eqref{eq:aux1_RFFapprox}, we get 
	\begin{equation}
		\norm{\mu_\mbp - \mu_\mbq }_\cH \leq  \norm{\mu_\mbp^\Gamma - \mu_\mbq^\Gamma }_{\HFF} + C \sqrt{\frac{\log \delta^{-1}}{D}}\,,
	\end{equation}
	which concludes the proof.
\end{proof}

\begin{lemma}\label{lem:approx_mmd_rff}
	Let $\mbq$, $\mbp_1, \ldots, \mbp_B$ some distributions on $\cZ$, $\mu_\mbq^\Gamma$, $\mu_{\mbp_k}^\Gamma$ their respective KME in the random Fourier features space $\HFF$ of dimension $D$ for the kernel $\kappa$. Let $\omvect^\Gamma$ some weights in the Simplex $\Spx_B$ depending on the features $\Gamma$. Then:
	\begin{equation}
		\ee[3]{\Gamma}{ \MMD^2_\kappa \paren[2]{ \sum_{k=1}^{B} \omega_k^\Gamma \mbp_k,\mbq}} \leq \ee[3]{\Gamma}{ \norm[2]{ \sum_{k=1}^{B} \omega_k^\Gamma \mu_{\mbp_k}^\Gamma - \mu_\mbq^\Gamma}_{\HFF}^2 }+ C\sqrt{\frac{\log B}{D}}\,.
	\end{equation}
\end{lemma}

\begin{proof}
	We reuse the notations of the proof of Lemma~\ref{lem:approx_kme_rff}.
	Let us first remark that with probability at least $1-\delta$, for any distributions $\mbp,\mbq$:
	\begin{equation}
		\inner{ \mu_\mbq^\Gamma,\mu_\mbp^\Gamma}_{\HFF} = \frac{1}{D} \sum_{i=1}^{D} \ee{X\sim \mbp,Y \sim \mbq}{\phi(\gamma_i,X)\phi(\gamma_i,Y)}  \geq \ee{X\sim \mbq, Y \sim \mbq}{\kappa(X,Y)} - C \sqrt{\frac{\log \delta^{-1}}{D}}\,.
	\end{equation}
	This is obtained by remarking that $\ee{\Gamma}{\ee{X\sim \mbp,Y \sim \mbq}{\phi(\gamma_i,X)\phi(\gamma_i,Y)}} = \ee{X\sim \mbp, Y \sim \mbq}{\kappa(X,Y)} = \inner{\mu_\mbp,\mu_\mbq}$ and are upper bounded by construction of the random Fourier features (see Section~\ref{sec:RFF}) and by applying Hoeffding's inequality. It follows, after an union bound over all pair of distributions, with probability at least $1-\delta$:
	\begin{align*}
		&\norm[2]{ \sum_{k=1}^{B} \omega_k^\Gamma \mu_{\mbp_k}^\Gamma - \mu_\mbq^\Gamma}_{\HFF}^2 = \sum_{k, \ell =1}^{B} \omega_k^\Gamma \omega _\ell ^\Gamma \inner{ \mu_{\mbp_k}^\Gamma - \mu_\mbq^\Gamma,  \mu_{\mbp_\ell}^\Gamma - \mu_\mbq^\Gamma }_{\HFF} \\
		&= \sum_{k, \ell =1}^{B} \omega_k^\Gamma \omega _\ell ^\Gamma \paren{\inner{ \mu_{\mbp_k}^\Gamma,\mu_{\mbp_\ell}^\Gamma}_{\HFF}  -\inner{ \mu_{\mbp_k}^\Gamma, \mu_\mbq^\Gamma}_{\HFF} -\inner{ \mu_{\mbp_\ell}^\Gamma, \mu_\mbq^\Gamma}_{\HFF} +\inner{ \mu_{\mbq}^\Gamma, \mu_\mbq^\Gamma}_{\HFF} }\\
		& \geq \sum_{k, \ell =1}^{B} \omega_k^\Gamma \omega _\ell ^\Gamma \paren{\inner{ \mu_{\mbp_k},\mu_{\mbp_\ell}}_{\cH}  -\inner{ \mu_{\mbp_k}, \mu_\mbq}_{\cH} -\inner{ \mu_{\mbp_\ell}, \mu_\mbq}_{\cH} +\inner{ \mu_{\mbq}, \mu_\mbq}_{\cH} }  - C \sqrt{\frac{\log ((B+1) \delta^{-1})}{D}}\\
		&= \MMD^2_\kappa \paren[2]{ \sum_{k=1}^{B} \omega_k^\Gamma \mbp_k,\mbq}  - C \sqrt{\frac{\log ((B+1) \delta^{-1})}{D}}\,.
	\end{align*}
	As the inequality is satisfied with probability $\delta$ for any $\delta>0$, we can integrate and obtain that for some constant $C$:
	\begin{equation*}
		\ee[3]{\Gamma}{ \MMD^2_\kappa \paren[2]{ \sum_{k=1}^{B} \omega_k^\Gamma \mbp_k,\mbq}} \leq \ee[3]{\Gamma}{ \norm[2]{ \sum_{k=1}^{B} \omega_k^\Gamma \mu_{\mbp_k}^\Gamma - \mu_\mbq^\Gamma}_{\HFF}^2 }+ C\sqrt{\frac{\log B}{D}}\,.
	\end{equation*}
	A careful reader can notice that this inequality is an equality when $B=1$, which allows to upper bound $\log(B+1)$ by $C\log B$ for $B\geq 2$ and some absolute constant $C>0$.
\end{proof}

\subsection{Covariance operators}\label{APPssec:covop}
\begin{definition}[Covariance operator]\label{def:covop}
	Let $\mbq$ a squared integrable distribution in an Hilbert space $\cH$. The covariance operator of $\mbq$ is defined as:
	\begin{equation}
		\begin{cases}
			\cH & \rightarrow \cH \\
			v & \mapsto  \ee[1]{\Phi\sim \mbq}{ \inner{v,\Phi}_\cH \Phi} - \inner[1]{ \ee{\Phi \sim \mbq}{\Phi},v}_\cH\ee{\Phi \sim \mbq}{\Phi} =  \ee[1]{\Phi\sim \mbq}{ \inner{v,\Phi-\ee{\Phi \sim \mbq}{\Phi}}_\cH \Phi} \\
		\end{cases}
	\end{equation}
	By abuse of language, we refer to the \textbf{covariance of a distribution $\mbp$ on $\cZ$} as the covariance operator of the pushforward $\phi_\kappa(\mbp)$ of $\mbp$ into the RKHS $\cH$, that is, of the distribution $\mbq$ on $\cH$ where $\Phi = \phi_\kappa(Z) \sim \mbq$.
\end{definition}
For sake of clarity, we recall briefly below the definition of the trace and operator norm of a Hilbert-Schmidt operator. which aligns with the definition in finite dimension.

\begin{definition}[Trace and operator norm]\label{def:trace_opnorm}
	Let $\Sigma : \cH \to \cH$ be a trace class \batodo{trace class? Why not simply operator over...} operator over a separable Hilbert space $(\cH,\inner{\cdot,\cdot}_\cH)$. Then:
	\begin{equation}
		\norm{\Sigma}_{op} := \sup_{h: \norm{h}_\cH =1} \norm{\Sigma h }_\cH\,, \quad \text{and} \quad \tr \Sigma := \sum_{k=1}^{\infty} \inner{e_k,\Sigma e_k}_\cH\,,
	\end{equation}
	where $(e_k)_k$ is a countable orthonormal basis.
\end{definition}

\begin{lemma}\label{lem:trace}
	Let $ \Sigma, S$ the respective covariance operators of centered distribution $\mbp$ and $\mbq$. Then $\tr \Sigma^2:= \ee{Z,Z' \sim \mbp}{ \inner{Z,Z'}^2}$ and $\tr \Sigma S = \ee{Z \sim \mbp, Z'\sim \mbq}{ \inner{Z,Z'}^2}$.
\end{lemma}

\begin{proof}
	Let us prove the two statements simultaneously. 
	As $\Sigma$ and $S$ are Hilbert Schmidt operators, for $(e_k)_{k \geq 1}$ an orthonormal basis of $\cH$:
	\begin{equation}
		\tr \Sigma S = \sum_{k=1}^{\infty} \inner{  e_k, \Sigma Se_k}_\cH = \sum_{k=1}^\infty \inner{\Sigma e_k,S e_k}_\cH
		= \ee{Z\sim \mbp, Z'\sim \mbq}{ \sum_{k=1}^{\infty} \inner{Z,e_k}\inner{Z',e_k} \inner{Z,Z'}}\,,	
	\end{equation}
	using the linearity of the expectation. Then using again Parseval's identity, we get:
	\begin{equation*}
		\tr \Sigma S = \ee{Z\sim \mbp, Z'\sim \mbq}{\inner{Z,Z'}^2}.
	\end{equation*}
	Taking $S = \Sigma$ gives the second statement.
\end{proof}

\begin{lemma}{Trace of covariance operator of RFF}\label{lem:trace_RFF}
	Let $\Sigma^\Gamma$ the covariance operator of the random Fourier features $\phi^\Gamma(Z)\in \HFF$ for $Z\sim \mbp$. Then 
	\begin{equation}
		\ee{\Gamma}{ \tr \Sigma^\Gamma} = \tr \Sigma\,,
	\end{equation} 
	where $\Sigma$ is the covariance operator of the mapping $\phi_\kappa(Z)=\kappa(Z,\cdot)$ in the original RKHS $\cH$.
\end{lemma}

\begin{proof}
	Let us just use that for any covariance operator $\Sigma$ of a distribution $\mbp$, $\tr \Sigma = \ee{Z,Z'\sim \mbp}{\norm{Z-Z'}^2}/2$. It follows:
	\begin{align*}
		2\ee{\Gamma}{ \tr \Sigma^\Gamma} &= \ee{\Gamma}{\ee{Z,Z'\sim \mbp}{ \norm{\phi^\Gamma(Z) - \phi^{\Gamma}(Z')}^2_{\HFF}}} \\
		&=  \ee{Z,Z'\sim \mbp}{ \ee{\Gamma}{ \norm{\phi^\Gamma(Z)}^2_{\HFF} -2\inner{ \phi^\Gamma(Z),\phi^\Gamma(Z')}_{\HFF} + \norm{ \phi^{\Gamma}(Z')}^2_{\HFF}}}\\
		&= \ee{Z,Z'\sim \mbp}{ \kappa(Z,Z) - 2 \kappa(Z,Z')+ \kappa(Z',Z')} = 2\ee{Z,Z'\sim \mbp}{ \norm{ \kappa(Z,\cdot) - \kappa(Z',\cdot)}^2_\cH} = 2\tr \Sigma\,,
	\end{align*}
	which concludes the proof.
\end{proof}

\begin{lemma}[Expectation of operator norm]\label{lem:expect_opnorm}
	Let $\phi_1,\ldots, \phi_D$ i.i.d. centered random vectors in an Hilbert space $ \cH$. Let 
	\begin{equation}
		\widehat{\Sigma}(\cdot) = \frac{1}{D} \sum_{i=1}^{D} \inner{\cdot,\phi_i}_\cH\phi_i,
	\end{equation} 
	be the empirical covariance operator. Then:
	\begin{equation*}
		\norm{\Sigma}_{op} \leq	\e{ \norm[1]{\widehat{\Sigma}}_{op}} \leq \norm{\Sigma}_{op}  + \sqrt{\frac{ \ee[1]{\phi \sim \mbp}{ \norm{\phi}^4_\cH}}{D}}
	\end{equation*}
\end{lemma}

\begin{proof}
	Using triangle inequality, we have the following inequality:
	\begin{equation}
		\norm[1]{\widehat{\Sigma}_{op}} \leq \norm{\Sigma}_{op} + \norm{\widehat{\Sigma} - \Sigma}_{op} \leq  \norm{\Sigma}_{op} + \sqrt{\tr \paren[1]{\widehat{\Sigma} - \Sigma}^2 }.
	\end{equation}
	Let us now bound the trace of the square of the operators. Let $\widehat{\mbp} = \frac{1}{D} \sum_{i=1}^{D} \delta_{\phi_i}$ the empirical distribution. Then $\widehat{\Sigma}$ is the covariance operator of $\widehat{\mbp}$. It follows using Lemma~\ref{lem:trace}:
	\begin{align*}
		\tr\paren{ (\widehat{\Sigma}- \Sigma)^2} &= \tr \paren{ \widehat{\Sigma}^2 - \widehat{\Sigma}\Sigma - \Sigma \widehat{\Sigma} + \Sigma^2} = \ee{\phi,\phi' \sim \widehat{\mbp}}{\inner{\phi,\phi'}_\cH^2} - 2\ee{\phi\sim \mbp ,\phi' \sim \widehat{\mbp}}{\inner{\phi,\phi'}_\cH^2} + \tr \Sigma^2\\
		&= \frac{1}{D^2}\sum_{i,j=1}^D \inner{\phi_i,\phi_j}_\cH^2 - \frac{2}{D} \sum_{i=1}^{D} \ee{\phi}{\inner{ \phi_i, \phi}^2}_\cH + \tr \Sigma^2\,.
	\end{align*}
	Let us now take the expectation over the $\phi_i$-s. We get:
	\begin{align*}
		\e{	\tr (\widehat{\Sigma}- \Sigma)^2) } = \frac{1}{D} \ee{\phi \sim \mbp}{ \norm{\phi}^4_\cH} + \tr \Sigma^2 \paren{ \frac{D(D-1)}{D^2} - 2 +1} \leq \frac{ \ee{\phi \sim \mbp}{ \norm{\phi}^4_\cH}}{D}
	\end{align*}
	We conclude by Jensen's inequality.
\end{proof}

\begin{lemma}\label{lem:approx_opnorm_rff}
	Let $G$ the covariance of $\phi_\Gamma(Z) \in \mbr^D$ defined by $G =  \ee{Z}{(\phi_\Gamma(Z) - \mu)^T(\phi_\Gamma(Z) - \mu)}$ with $\mu = \mu(\Gamma) = \ee{Z}{\phi_\Gamma(Z)}$. Then:
	\begin{equation}
		\norm{\Sigma}_{op} \leq	\ee{\Gamma}{ \norm{G}_{op}} \leq  \norm{\Sigma}_{op} + \sqrt{\frac{\tr \Sigma}{D}}\,,
	\end{equation}
	where $\Sigma : L^2(\mbp) \to L^2(\mbp)$ is the covariance operator of $\phi(Z) = k(Z,\cdot) \in \cH$.
\end{lemma}

\begin{proof}
	Let us first remark that $G$ is a Gram matrix of vectors of $L^2(\mbp)$. Let us first denote $\phi_{\Gamma}(z) = \paren{\phi(\gamma_i,Z) }_{i=1}^D$, $\mu_i = \ee{Z\sim \mbp}{\phi(\gamma_i,z) }$ and $\bar{\phi}_i(\cdot) = \bar{\phi}(\gamma_i,\cdot)-\mu_i$. Then for any $i,j \in \intr{D}$:
	\begin{equation*}
		G_{ij} = \ee{Z}{(\phi(\gamma_i,Z) -m_i)(\phi(\gamma_j,Z) - m_j) } := \ee{Z}{\bar{\phi}_i(Z) \bar{\phi}_j(Z) } = \inner{\bar{\phi}_i,\bar{\phi}_j }_{L^2(\mbp)} 
	\end{equation*}
	Then:
	\begin{align*}
		\norm{G}_{op} = \sup_{ \norm{u}_D=1} u^TGu = \sup_{ \norm{u}_D=1} \norm[3]{ \sum_{i=1}^{D} u_i \bar{\phi}_i }^2_{L^2(\mbp)} = \sup_{ \norm{u}_D=1} \sup_{ \norm{v}_{L^2(\mbp)}=1}\inner[3]{ v,  \sum_{i=1}^{D} u_i \bar{\phi}_i}^2_{L^2(\mbp)} \,,
	\end{align*}
	where the second supremum is taken over the unit ball of $L^2(\mbp)$. We can invert the supremums and use the symmetry: 
	\begin{align*}
		\norm{G}_{op} &=  \sup_{ \norm{v}_{L^2(\mbp)}=1}  \sup_{ \norm{u}_D=1} \paren[3]{ \sum_{i=1}^{D} u_i\inner{ v,   \bar{\phi}_i}_{L^2(\mbp)}}^2 =  \sup_{ \norm{v}_{L^2(\mbp)}=1} \paren[3]{ \sup_{ \norm{u}_D=1} \sum_{i=1}^{D} u_i\inner{ v,   \bar{\phi}_i}_{L^2(\mbp)}}^2 \\
		&= \sup_{ \norm{v}_{L^2(\mbp)}=1} \sum_{i=1}^{D} \inner{v ,\bar{\phi}_i }^2_{L^2(\mbp)} = \sup_{ \norm{v}_{L^2(\mbp)}=1} \inner{v, \widehat{\Sigma} v }_{L^2(\mbp)}= \norm[1]{\widehat{\Sigma}}_{op},
	\end{align*}
	where $\widehat{\Sigma}$ is the empirical covariance operator defined by:
	\begin{equation*}
		\widehat{\Sigma} :
		\begin{cases}
			L^2(\mbp) & \rightarrow  L^2(\mbp) \\
			v & \mapsto \frac{1}{D}\sum_{i=1}^{D} \sqrt{D}\bar{\phi}_i \inner{v,\sqrt{D}\bar{\phi}_i}_{L^2(\mbp)} \,. \\
		\end{cases}
	\end{equation*}
	Effectively, we have that for any $v \in \cH$ and $z\in \cZ$:
	\begin{align*}
		\ee{\Gamma}{ \widehat{\Sigma}v(z)}& = D\ee{ \Gamma, Z}{ \phi_1(z) v(Z)(\phi_1(Z)-\mu_1)} \\
		&= \ee{Z}{ v(Z) \ee{\Gamma}{ D\phi_1(z)\phi_1(Z)}} - \ee{Z,Z'}{ v(Z) \ee{\Gamma}{ D\phi_1(z)\phi_1(Z')}} \\
		&= \ee{Z}{v(Z)\kappa(Z,z)} - \ee{Z,Z'}{ v(Z) \kappa(z,Z')} \\
		& = \ee{Z}{\inner{ \kappa(Z,\cdot),v}_\cH \kappa(Z,z)} - \inner{v , \mu_\mbp } \mu_\mbp(z)\,,
	\end{align*}
	where we recognize the covariance operator. %As $\cH$ is dense in $L^2(\mbp)$ \citep[Theorem 4.26]{steinwart2008support}, the operator norm coincides. \todo{Attention y'a peut-être des subtilités}
	We then use Lemma~\ref{lem:expect_opnorm}:
	\begin{equation*}
		\e{ \norm{G}_{op}} = \e{ \norm[1]{\widehat{\Sigma}}_{op}} \leq \e{ \norm[1]{\Sigma}_{op}}+ \sqrt{\frac{\ee{\Gamma}{ \norm[1]{ \sqrt{D} \bar{\phi}_1}^4_{L^2(\mbp)}}}{D}}\,.
	\end{equation*}
	We conclude using that $\sqrt{D} \bar{\phi}_1$ is upper bounded by $1$ and that $\ee{\Gamma}{ \norm[1]{ \sqrt{D} \bar{\phi}_1}^2_{L^2(\mbp)}}= \tr \Sigma$.
	\begin{remark}
		We have used that the spectrum of the covariance operator $\Sigma$ in $L^2(\mbp)$ and $\cH$ are the same \citep[Proposition 8]{rosasco2010learning}.
	\end{remark}
	
\end{proof}

\section{Proofs}\label{APPsec:proofs}

This section provides detailed proofs for the theoretical results outlined in the main paper.

\subsection{Proof of Example~\ref{ex:linear_regression_rkhs}}
Let us recall that $\ell_{(\alpha,\beta)}((x,y)) = \paren{ \inner{x,\alpha}+\beta -y }^2$. If $\beta \neq 0$, then $\ell_{(\alpha,\beta)}(\cdot) = \beta^2 \kappa\paren{ (\alpha/\beta,-1/\beta),\cdot}$. For $\beta = 0$:
\begin{equation*}
	\ell_{(\alpha,\beta)}((x,y)) = \paren{ \inner{x,\alpha} -y }^2 = \frac{1}{2}\kappa\paren{ (\alpha,-1), (x,y)} + \frac{1}{2}\kappa\paren{ (-\alpha,1), (x,y)} - \kappa\paren{0,(x,y)}
\end{equation*}
\qed
\subsection{Proof of Lemma~\ref{lem:excessrisk_mmd}}

\textbf{Upper bound.}
As $\ell_\theta(\cdot) = c_\theta + h_\theta(\cdot)$ by assumption, we get, for any $\theta \in \Theta$:
\begin{equation*}
	\sum_{k=1}^{B}\widehat{\omega}_k \widehat{\cR}_k(\theta) - \cR_1(\theta) = 	\sum_{k=1}^{B}\widehat{\omega}_k \ee{Z\sim \widehat{\mbp}_k}{h_\theta(Z)} - \ee{Z\sim \mbp_1}{h_\theta(Z)} =  \ee{Z\sim \widehat{\mbp}(\widehat{\omega})}{h_\theta(Z)} - \ee{Z\sim\mbp_1}{h_\theta(Z)}\,,
\end{equation*}
where $\widehat{\mbp}(\widehat{\omega}) = \sum_{k=1}^B \widehat{\omega}_k \widehat{\mbp}_k$. It follows, by definition of the MMD (Eq.~\eqref{eq:MMD_def}):
\begin{equation*}
	\sum_{k=1}^{B}\widehat{\omega}_k \widehat{\cR}_k(\theta) - \cR_1(\theta) \leq \norm{h_\theta}_\cH \MMD \paren{\widehat{\mbp}(\widehat{\omega}), \mbp_1}\,,
\end{equation*}
which leads to \eqref{eq:generror_to_mmd}.

\textbf{Lower bound.}
Similarly as above, for any $\theta \in \Theta$:
\begin{align*}
	\widehat{\cR}_1(\theta) - \cR_1(\theta) =  \ee{Z\sim \widehat{\mbp}_1}{h_\theta(Z)} - \ee{Z\sim\mbp_1}{h_\theta(Z)} = \inner{ h_\theta , \widehat{\mu}_1 - \mu_1}_\cH\,.
\end{align*}
As $\set{ h \in \cH : \norm{h}_\cH =r}  \subset \set{h_\theta}_{\theta \in \Theta}$, there exists $\theta \in \Theta$ such that $h_\theta = r\frac{\widehat{\mu}_1 - \mu_1}{\norm{\wh{\mu}_1 - \mu_1}_\cH}$. It follows:
\begin{align*}
	\e{ \sup_{\theta \in \Theta} 	\paren{\widehat{\cR}_1(\theta) - \cR_1(\theta)}^2 } \geq r^2 \e{ \norm{ \widehat{\mu}_1 - \mu_1}^2_\cH} = r^2\frac{\tr \Sigma_1}{n_1}\,,
\end{align*}
which concludes the proof. \qed
\subsection{Proof of Theorem~\ref{prop:kme_aggregation}}

We begin by restated a general result of \citet{blanchard2024estimation} on the Q-aggregation method (Algorithm~\ref{alg:Qaggregation}). It is adapted from Eq. (90) p.62 of the proof of Theorem~3 of this work.

\begin{theorem}[\citealp{blanchard2024estimation}, restated]\label{thm:oracle_ineq_BS}
	Let $u_0 \geq 2\log BN_1$, $\wh{\omvect}$ be the output of Algorithm~\ref{alg:Qaggregation} for $\{Z_i^{(1)}\}_{i=1}^{n_1}$ a sample of $\mbp_1$, $C_Q^2,C_p \geq C_0u_0$, $\widehat{\mu}_k = n_k^{-1}\sum_{i=1}^{n_k} Z_i^{(k)}$ the empirical means of i.i.d. samples of distribution $\mbp_k$. Assume that all the distributions are bounded by $M$, then
	\begin{align}
		\e[3]{\norm[2]{ \sum_{k=1}^{B} \widehat{\omega}_k \widehat{\mu}_k - \mu_1}_\cH^2} \leq& \min_{\omvect \in \Spx_B} \brac[3]{
			R_1(\omvect) +C\sqrt{u_0} Q_1(\omvect) + \frac{C M u_0}{n_1} \sum_{k=2}^{B} \omega_k \paren[2]{ \|\mu_k-\mu_1\|+\sqrt{\frac{\tr \Sigma_k}{n_k} }} }\notag\\
		&+ C \frac{\sqrt{\tr \Sigma_1^2}}{n_1}u_0 + CM\frac{\sqrt{\tr \Sigma_1}}{N_1^{\nicefrac{3}{2}}}u_0+  C\frac{M^2}{N^2_1}u_0 ^2\,,  \label{al:thmbnd_oracle}
	\end{align}
	where $C>0$ is an absolute constant depending on $C_0$, $\Sigma_k$ is the covariance of $\mbp_k$, $n_k$ the sample size of sample $k$ and
	\begin{gather}
		R_1(\omvect) = \norm[3]{ \sum_{k=1}^{B}\omega_k(\mu_k - \mu_1)}^2_\cH + \sum_{k=1}^{B}\omega_k^2 \frac{\tr \Sigma_k}{n_k}\,, \label{eq:R1_def} \\
		Q(\omvect) = \frac{1}{\sqrt{n_1}}\sum_{k=2}^{B} \omega_k \sqrt{\inner[1]{\mu_1-\mu_k, \Sigma_1(\mu_1-\mu_k)}_\cH + \frac{\tr \Sigma_1\Sigma_k}{n_k}} \label{eq:Q1_def}
	\end{gather}
\end{theorem}

The objective is to follow the proof of Theorem~3 of \citet{blanchard2024estimation} and adapt it using the context that the estimated quantities are KMEs, by using following Lemma~\ref{lem:bound_trace} to simplify the bound and the assumptions of the Q-aggregation method.
\begin{lemma}\label{lem:bound_trace}
	Let $\kappa$ a kernel constant over the diagonal ($k(x,x) = M^2, \forall x$). Let $\mbp$ and $\mbq$ two distributions and $\Sigma_\mbp$ and $\Sigma_\mbq$ their respective covariance operators in the RKHS. Then:
	\begin{equation}
		\abs{ \tr \Sigma_\mbp - \tr \Sigma_\mbq } \leq 2M\norm{ \mu_\mbp - \mu_\mbq}_\cH.
	\end{equation}
\end{lemma}

\begin{proof}
	Let us first remark that:
	\begin{equation*}
		\tr \Sigma_\mbp = \ee{X \sim \mbp}{\norm{X-\mu_{\mbp}}^2_\cH}= \e{k(X,X)} - \norm{\mu_\mbp}^2_\cH = M^2 -  \norm{\mu_\mbp}^2_\cH\,.
	\end{equation*}
	Then 
	\begin{align*}
		\abs{ \tr \Sigma_\mbp - \tr \Sigma_\mbq } \leq \abs{ \norm{\mu_\mbp}^2_\cH-  \norm{\mu_\mbq}^2_\cH } &= \abs{ \inner{ \mu_\mbp, \mu_\mbp - \mu_\mbq}_\cH + \inner{ \mu_\mbp - \mu_\mbq, \mu_\mbq}_\cH} \\
		&\leq \paren{ \norm{\mu_\mbp}_\cH + \norm{\mu_\mbq}_\cH} \norm{ \mu_\mbq - \mu_\mbq}_\cH\,.
	\end{align*}
	We conclude using that the norms of the KMEs are bounded by M:
	\begin{align*}
		\norm{\mu_\mbp}^2_\cH = \e{\inner{ k(X,\cdot),k(X',\cdot)}_\cH} \leq \e{\norm{ k(X,\cdot)}_\cH \norm{k(X',\cdot)}_\cH} = \e{\sqrt{k(X,X)k(X',X')}} \leq M^2\,.
	\end{align*}
\end{proof}

\paragraph{Proof of Theorem~\ref{prop:kme_aggregation}}

According to Theorem~\ref{thm:oracle_ineq_BS}, we have for any $\omvect \in \Spx_B$:

\begin{align}
	\e{ \norm{\wh{\mu}_{\wh{\omvect}} - \mu_1}^2_\cH} &\leq R_1(\omvect) + C \sqrt{u_0} Q(\omvect) +\frac{Cu_0}{n_1} \sum_{k=2}^{B} \omega_k \paren{ \|\mu_k-\mu_1\|_\cH+ \sqrt{\frac{\tr \Sigma_k}{n_k}}} \notag\\
	&+ C\frac{\sqrt{\tr \Sigma_1^2}}{n_1}u_0 + C\frac{\sqrt{\tr \Sigma_1}}{n_1^{3/2}}u_0+  C\frac{u_0 ^2}{n_1^2}\,,  \label{al:eqaux1}
\end{align}
where $R_1$ and $Q_1$ are defined respectively in \eqref{eq:R1_def} and \eqref{eq:Q1_def}.
%	\begin{gather*}
	%		R_1(\omvect) = \norm{ \sum_{k=1}^{B}\omega_k(\mu_k - \mu_1)}^2_\cH + \sum_{k=1}^{B}\omega_k^2 \frac{\tr \Sigma_k}{n_k}\,, \\
	%		Q(\omvect) = \frac{1}{\sqrt{n_1}}\sum_{k=2}^{B} \omega_k \sqrt{\inner{\mu_1-\mu_k, \Sigma_1(\mu_1-\mu_k)}_\cH + \frac{\tr \Sigma_1\Sigma_k}{n_k}}
	%	\end{gather*}
By remarking that:
\begin{equation}\label{eq:boundQproof}
	Q(\omvect) \leq \sqrt{\frac{\norm{\Sigma_1}_{op}}{n_1}} \sum_{k=2}^{B}\omega_k\paren[2]{ \|\mu_k-\mu_1\|_\cH+\sqrt{\tr \Sigma_k/n_k} }\,,
\end{equation}
it remains to choose a weight $\omega$ to bound effectively this quantity and $R_1(\omvect)$.

Let $V$ a subset of agents. We first fix $\omega_k = 0$ for $k \notin V$. Then using Lemma~\ref{lem:bound_trace}:
\begin{align*}
	%	R_1(\omvect) \leq (1-\omega_1)^2\Delta^2 + \omega_1^2 \frac{\tr \Sigma_1}{n_1} + \sum_{k\in V_\Delta, k\neq 1} \frac{\omega_k^2}{n_k}\paren{\tr \Sigma_1+ 2 \Delta}\,.
	R_1(\omvect) \leq (1-\omega_1)^2\Delta^2 + \omega_1^2 \frac{\tr \Sigma_1}{n_1} + \sum_{k\in V, k\neq 1} \frac{\omega_k^2}{n_k}\paren{\tr \Sigma_1+ 2 \Delta}\,.
\end{align*}
where $\Delta = \Delta_V = \max_{k \in V} \norm{\mu_1-\mu_k}_\cH$.
Let us then choose:
\begin{equation*}
	\omega_1 = \frac{\Delta^2 + \frac{\tr \Sigma_1 + 2\Delta}{n_V -n_1}}{\Delta^2 + \frac{\tr \Sigma_1 + 2\Delta}{n_V - n_1} + \frac{\tr \Sigma_1}{n_1}}
	\,, \quad \omega_k = (1-\omega_1) \frac{n_k}{n_V - n_1}\,,
\end{equation*}
Then:
\begin{align*}
	R_1(\omvect) &\leq (1-\omega_1)^2\brac{\Delta^2 + \frac{\tr \Sigma_1 + 2\Delta}{n_V - n_1} } + \omega_1^2 \frac{\tr \Sigma_1}{n_1} \\
	& = \frac{\tr \Sigma_1}{n_1}\frac{\Delta^2 + \frac{\tr \Sigma_1 + 2\Delta}{n_V -n_1}}{\Delta^2 + \frac{\tr \Sigma_1 + 2\Delta}{n_V - n_1} + \frac{\tr \Sigma_1}{n_1}} \\
	& \leq \Delta^2 + \frac{\tr \Sigma_1}{n_1} \frac{\frac{\tr \Sigma_1 + 2\Delta}{n_V -n_1}}{ \frac{\tr\Sigma_1 n_V}{n_1(n_V - n_1)}}\\
	& \leq \Delta^2 + \frac{\tr \Sigma_1}{n_V} + \frac{2\Delta}{n_V}\,.
\end{align*}
For the same weights $\omvect$ we now bound \eqref{eq:boundQproof}. Firstly
\begin{align*}
	\sum_{k=2}^{B}\omega_k \norm{\mu_k - \mu_1}_\cH \leq \Delta (1-\omega_1) \leq \Delta\frac{\frac{\tr\Sigma_1}{n_1}}{\Delta^2 + 0 + \frac{\tr \Sigma_1}{n_1}} \leq \min \paren{ \Delta, \sqrt{\frac{\tr \Sigma_1}{n_1}}}\,.
\end{align*}
For the second part we get:
\begin{align*}
	\sum_{k=2}^{B} \omega_k \sqrt{\frac{\tr \Sigma_k}{n_k}} \leq \sqrt{\tr \Sigma_1 + 2\Delta} \sum_{k=2}^{B} \frac{\omega_k}{\sqrt{n_k}} = \sqrt{\tr \Sigma_1 + 2\Delta} (1-\omega_1) \frac{\sum_{k\in V, k\neq 1} \sqrt{n_k}}{n_V - n_1}.
\end{align*}
Let us first remark that:
\begin{equation}
	\sqrt{\tr \Sigma_1 + 2\Delta} \leq \sqrt{\tr \Sigma_1} + \frac{\Delta}{\sqrt{\tr \Sigma_1}}
	%\min \paren{\sqrt{2\Delta}, \frac{\Delta}{\sqrt{\tr \Sigma_1}}}
\end{equation}
and that, 
%Using that $(1-\omega_1) \leq \frac{n_\Delta-n_1}{n_\Delta}$ and
by concavity
\begin{equation*}
	\sum_{k\in V, k\neq 1} \sqrt{n_k} \leq \sqrt{(|V| -1)(n_V - n_1)}.
\end{equation*} 
As $(1-\omega_1) \leq \frac{n_V-n_1}{n_V}$, we get:
\begin{equation*}
	\sqrt{\tr \Sigma_1}(1-\omega_1) \frac{\sqrt{(|V| -1)(n_V - n_1)}}{n_V-n_1} \leq \sqrt{\tr \Sigma_1} \sqrt{\frac{(|V| -1)}{n_V}}\,.
\end{equation*}
%	%Let us now bound $(1-\omega_1)\sqrt{\Delta}$:
Using that 
\begin{align*}
	(1-\omega_1)\frac{\Delta}{\sqrt{\tr \Sigma_1}} &= \frac{\sqrt{\tr \Sigma_1}}{n_1} \frac{\Delta}{\Delta^2 + \frac{2\Delta}{n_V - n_1} + \frac{\tr \Sigma_1n_V}{n_1(n_V-n_1)}}
	%		&= \frac{n_V - n_1}{n_V} \frac{\sqrt{\Delta}}{\Delta^2 \frac{n_1(n_V - n_1)}{\tr \Sigma_1 n_V} + 2\Delta\frac{2n_1}{n_V \tr \Sigma_1} +1} \\
	%		& \leq  \frac{n_V - n_1}{n_V} \paren{\frac{\tr \Sigma_1 n_V}{n_1(n_V - n_1)} }^{1/4}\,,
	\leq \frac{\sqrt{\tr \Sigma_1}}{n_1} \sqrt{\frac{n_1(n_V - n_1)}{\tr \Sigma_1 n_V}} = \sqrt{\frac{n_V - n_1}{n_1 n_V}}
\end{align*}
%using Lemma~\ref{lem:technical1} for the last inequality. Then combining these two bounds leads to:\todo{missing argument here}
and again the concavity, we obtain:
\begin{equation}
	\sum_{k=2}^{B}\omega_k \sqrt{\frac{\tr \Sigma_k}{n_k}} \leq \sqrt{\frac{|V| -1}{n_V}}  \brac{ \sqrt{\tr \Sigma_1} + \frac{1}{\sqrt{n_1}}} .%\leq \sqrt{\frac{\tr \Sigma_1}{n_1}} + \frac{1}{n_1}\,,
\end{equation}
%where we have used that $\frac{|V_\Delta| -1}{n_V}\leq n_1^{-1}$\todo{wrong :(}.
We can now plug all the bound into \eqref{al:eqaux1}:
\begin{align}
	\e{ \norm{\wh{\mu}_{\wh{\omvect}} - \mu_1}^2_\cH} &\leq   \brac{\Delta^2 + \frac{\tr \Sigma_1}{n_V} + \frac{2\Delta}{n_V}}+ C \brac{ \sqrt{\frac{\norm{\Sigma_1}_{op}u_0}{n_1}} + \frac{u_0}{n_1} } \sqrt{\frac{\tr \Sigma_1}{n_1}} \\
	&+C \brac{ \sqrt{\frac{\norm{\Sigma_1}_{op}u_0}{n_1}} + \frac{u_0}{n_1} }\sqrt{\frac{|V| -1}{n_V}}  \brac{ \sqrt{\tr \Sigma_1} + \frac{1}{\sqrt{n_1}}}  \notag\\
	% &+C \brac{ \sqrt{\frac{\norm{\Sigma_1}_{op}u_0}{n_1}} + \frac{u_0}{n_1} } \frac{1}{n_1}  \notag\\
	&+ C\frac{\sqrt{ \norm{\Sigma_1}_{op}\tr \Sigma_1}}{n_1}u_0 + C\frac{\sqrt{\tr \Sigma_1}}{n_1^{3/2}}u_0+  C\frac{u_0 ^2}{n_1^2}\,.  \label{al:eqaux2}
\end{align}
After combining the terms and upper bounding $\norm{\Sigma_1}_{op} \leq \tr \Sigma_1 \leq 1$ we get:
\begin{align*}
	\e{ \norm{\wh{\mu}_{\wh{\omvect}} - \mu_1}^2_\cH} &\leq   \brac{\Delta^2 + \frac{\tr \Sigma_1}{n_V} + \frac{2\Delta}{n_V}} 
	+ C\frac{\sqrt{ \norm{\Sigma_1}_{op}\tr \Sigma_1}}{n_1}u_0 +  C\frac{u_0}{n_1^{\nicefrac{3}{2}}}+  C\frac{u_0 ^2}{n_1^2} \\
	% C\frac{\sqrt{\tr \Sigma_1}}{n_1}u_0 \max \paren[3]{\frac{1}{\sqrt{n_1}}, \sqrt{\frac{|V_\Delta|-1}{n_V}}}+  C\frac{u_0 ^2}{n_1^2} \\
	& \leq \brac{\Delta^2 + \frac{\tr \Sigma_1}{n_V} + \frac{2\Delta}{n_V}}
	+ \frac{Cu_0}{\sqrt{n_1}} \max\paren[3]{ \sqrt{\frac{|V|-1}{n_V}}, \frac{1}{\sqrt{n_1}}} \max \paren[3]{ \frac{\tr \Sigma_1}{\sqrt{d^e_1}},\frac{u_0}{\sqrt{n_1} }}\,.
\end{align*}
\qed
\subsection{Proof of Corollary~\ref{thm:kernel}}

Assume $\widehat{\theta} \in \argmin_{\theta \in \Theta} \sum_{k=1}^{B} \widehat{\omega}_k \frac{1}{n_k} \sum_{i=1}^{n_k} \ell_\theta(Z_i^{(k)})$, we neglect the optimization error. Then:
\begin{align*}
	\e[1]{\cR^{(1)}_{\widehat{\theta}}} = \e{\inner{ \ell_{\widehat{\theta}}, \mu_1}_\cH} 
	=  \e{\inner{ \ell_{\widehat{\theta}}, \mu_1 - \widehat{\mu}}_\cH + \inner{ \ell_{\widehat{\theta}}, \widehat{\mu}}_\cH }\,.  
\end{align*}
As for any $\theta$, $\inner{ \ell_{\theta}, \widehat{\mu}}_\cH  =\sum_{k=1}^{B} \widehat{\omega}_k \frac{1}{n_k} \sum_{i=1}^{n_k} \ell_{\theta}(Z_i^{(k)})$, we get by definition of $\widehat{\theta}$ that $\inner{ \ell_{\widehat{\theta}}, \widehat{\mu}}_\cH \leq \inner{ \ell_{\theta^*}, \widehat{\mu}}_\cH$. Using again that $\cR^{(1)}_{\theta^*} =  \inner{ \ell_{\theta^*}, \mu_1}_\cH$, it follows:
\begin{align*}
	\e[1]{\cR^{(1)}_{\widehat{\theta}}} - \cR^{(1)}_{\theta^*} &\leq \e{ \inner{ \ell_{\widehat{\theta}}, \mu_1 - \widehat{\mu}}_\cH + \inner{ \ell_{\theta^*}, \widehat{\mu} - \mu_1}_\cH} \\
	& \leq \e{ \paren{ \norm{\ell_{\hat{\theta}} }_\cH + \norm{\ell_{\theta^*} }_\cH }\norm{\widehat{\mu}- \mu_1}_\cH} \\
	&\leq 	2\sup_{\theta \in \Theta} \norm{\ell_\theta}_{\cH}  \e{ \norm{ \mu_1 - \widehat{\mu}}_{\cH}}\,,
\end{align*}
thanks to Cauchy-Schwartz inequality. \qed

\subsection{Proof of Theorem~\ref{cor:kmerff}}

By combining Lemma~\ref{lem:excessrisk_mmd} and Equation~\eqref{eq:approx_err_to_excess_risk}, we know that controlling the MMD distance between the empirical mixture to $\mbp_1$ leads to a control of the excess risk. Let us control this quantity.

From Lemma~\ref{lem:approx_mmd_rff}, applied conditionally to the datasets $\cD$ to $\mbp_k  \leftarrow \widehat{\mbp}_k$ and $\mbq \leftarrow \mbp_1$, we first have that:
\begin{equation}
	\ee[3]{\Gamma,\cD}{ \MMD^2_\kappa \paren[2]{ \sum_{k=1}^{B} \omega_k^\Gamma \widehat{\mbp}_k,\mbp_1}} \leq \ee[3]{\Gamma,\cD}{ \norm[2]{ \sum_{k=1}^{B} \omega_k^\Gamma \wh{\mu}_{k}^\Gamma - \mu_1^\Gamma}_{\HFF}^2 }+ C\sqrt{\frac{\log B}{D}}\,.
\end{equation}
We can then apply Theorem~\ref{thm:kernel} in the random Fourier features RKHS $\HFF$, so conditionally to $\Gamma$. For any subset $V \subset \intr{B}$:
\begin{align}\label{al:corkmerff_aux1}
	\ee[3]{\cD}{ \norm[2]{ \sum_{k=1}^B \wh{\omega}^\Gamma_k \wh{\mu}^\Gamma_k - \mu^\Gamma_1}^2_\cH |\Gamma } &\leq \brac{\Delta^2_{\Gamma,V} + \frac{\tr \Sigma^\Gamma_1+2 \Delta_{\Gamma,V}}{n_V}} \notag\\
	&+ \frac{Cu_0}{\sqrt{n_1}} \max\paren[3]{\sqrt{\frac{|V|-1}{n_V}}, \frac{1}{\sqrt{n_1}}} \max \paren{ \frac{\tr \Sigma^\Gamma_1}{\sqrt{d^{e,\Gamma}_1}},\frac{u_0}{\sqrt{n_1} }}\,,
\end{align}
where $\Sigma_1^\Gamma$ is the covariance operator of the random Fourier features $\phi_\Gamma(Z)$ for $Z\sim \mbp_1$, $d^{e,\Gamma}_1$ is its effective dimension and $\Delta_{\Gamma,V} = \max_{k \in V} \norm{\mu^\Gamma_{\mbp_k} - \mu^\Gamma_{\mbp_1}}_{\HFF}$.

Using Lemma~\ref{lem:trace_RFF}, we have $\ee{\Gamma}{\tr \Sigma_1^\Gamma} = \tr \Sigma_1$, and, using Jensen's inequality and Lemma~\ref{lem:expect_opnorm}:
\begin{align*}
	\ee{\Gamma}{ \frac{\tr \Sigma^\Gamma_1}{\sqrt{d^{e,\Gamma}_1}} } = \ee{\Gamma}{ \sqrt{\tr \Sigma^\Gamma_1 \norm{\Sigma^\Gamma_1}_{op}}} \leq \sqrt{\tr \Sigma_1\ee{\Gamma}{\norm{\Sigma^\Gamma_1}_{op} }} \leq \sqrt{\tr \Sigma_1}\sqrt{ \norm{\Sigma_1}_{op} + \sqrt{\frac{\tr \Sigma_1}{D}}}\,.
\end{align*} 
Using that $\sqrt{a^2+b} \leq a + b/(2a)$ for $a,b\geq 0$, we obtain that:
\begin{equation}
	\ee{\Gamma}{ \frac{\tr \Sigma^\Gamma_1}{\sqrt{d^{e,\Gamma}_1}} } \leq \sqrt{\tr \Sigma_1}\paren{ \norm{\Sigma_1}_{op}^{\nicefrac{1}{2}} + C\sqrt{\frac{d^e_1}{D}}} \leq \frac{\tr \Sigma_1}{\sqrt{d_1^e}} + C\sqrt{\frac{d^e_1}{D}}\,,
\end{equation}
using that $\tr \Sigma_1 $ is upper bounded by $1$ since the kernel is upper bounded by $1$.

It remains to control $\Delta_{\Gamma,V}$. Using Lemma~\ref{lem:approx_kme_rff} combined with an union bound over the $B$ agents, with probability at least $1-e^{-u}$, for $u\geq 0$:
\begin{equation}
	\Delta_{\Gamma,V} \leq \max_{k \in V_\Delta} \norm{\mu_{\mbp_k} - \mu_{\mbp_1}}_{\cH} + C \sqrt{\frac{u\log B}{D}} =\Delta_V + C \sqrt{\frac{u\log B}{D}}\,,
\end{equation}
where $\Delta_V = \sup_{k\in V} \MMD(\mbp_1,\mbp_k)$. It follows that
\begin{align*}
	\ee{\Gamma}{\Delta_{\Gamma,V}^2 + \frac{\tr \Sigma^\Gamma_1+2 \Delta^\Gamma}{n_V}}\leq \Delta_V^2 + \frac{\tr \Sigma_1+2 \Delta_V}{n_V} + C \sqrt{\frac{\log B}{D}}\,.
\end{align*}
using that $\Delta_V \leq 1$. Combining the two upper bounds leads to the result.
\qed

\end{document}